\def\year{2022}\relax
\documentclass[letterpaper]{article} 
\usepackage{aaai22}  

\usepackage{times}  
\usepackage{helvet}  
\usepackage{courier}  
\usepackage[hyphens]{url}  
\usepackage{graphicx} 
\usepackage{natbib}  
\usepackage{caption} 
\DeclareCaptionStyle{ruled}{labelfont=normalfont,labelsep=colon,strut=off} 
\usepackage{algorithm}
\usepackage{algorithmic}

\usepackage{booktabs}

\usepackage{graphicx}

\usepackage{amsmath,amsfonts,bm}

\def\eqref#1{equation~\ref{#1}}

\def\1{\bm{1}}

\DeclareMathAlphabet{\mathsfit}{\encodingdefault}{\sfdefault}{m}{sl}
\SetMathAlphabet{\mathsfit}{bold}{\encodingdefault}{\sfdefault}{bx}{n}

\DeclareMathOperator*{\argmin}{arg\,min}

\usepackage{amsthm}

\theoremstyle{remark}

\newtheorem{theorem}{Theorem}
\newtheorem{lemma}{Lemma}

\title{How to Distribute Data across Tasks for Meta-learning?}
\author{
    Alexandru Cioba, Michael Bromberg, Qian Wang, Ritwik Niyogi, \\Georgios Batzolis, Jezabel Garcia, Da-shan Shiu, Alberto Bernacchia
}
\affiliations{
    MediaTek Research
}
\usepackage{bibentry}

\begin{document}

\maketitle

\begin{abstract}
Meta-learning models transfer the knowledge acquired from previous tasks to quickly learn new ones. They are trained on benchmarks with a fixed number of data points per task. This number is usually arbitrary and it is unknown how it affects performance at testing. Since labelling of data is expensive, finding the optimal allocation of labels across training tasks may reduce costs. Given a fixed budget of labels, should we use a small number of highly labelled tasks, or many tasks with few labels each? Should we allocate more labels to some tasks and less to others?
We show that: 1) If tasks are homogeneous, there is a uniform optimal allocation, whereby all tasks get the same amount of data; 2) At fixed budget, there is a trade-off between number of tasks and number of data points per task, with a unique solution for the optimum; 3) When trained separately, harder task should get more data, at the cost of a smaller number of tasks; 4) When training on a mixture of easy and hard tasks, more data should be allocated to easy tasks. Interestingly, Neuroscience experiments have shown that human visual skills also transfer better from easy tasks. We prove these results mathematically on mixed linear regression, and we show empirically that the same results hold for few-shot image classification on CIFAR-FS and mini-ImageNet. Our results provide guidance for allocating labels across tasks when collecting data for meta-learning.

\end{abstract}
\section{Introduction}

Deep learning (DL) models require a large amount of data in order to perform well, when trained from scratch, but labeling data is expensive and time consuming.
An effective approach to avoid the costs of collecting and labeling a large amount of data is transfer learning: train a model on one big dataset, or a few related datasets that are already available, and then fine-tune the model on the target dataset, which can be of much smaller size \cite{donahue_decaf:_2014}.
In this context, there has been a recent surge of interest in the field of \emph{meta-learning}, which is inspired by the ability of humans to \emph{learn how to learn} \cite{hospedales_meta-learning_2020}. 
A model is \emph{meta-trained} on a large number of tasks, each characterized by a small dataset, and \emph{meta-tested} on the target dataset.  

The number of data points per task is usually set to an arbitrary number in standard meta-learning benchmarks.
For example, in few-shot image classification benchmarks, such as \emph{mini}-ImageNet \cite{vinyals_matching_2017}, \cite{ravi_optimization_2017} and CIFAR-FS \cite{bertinetto_meta-learning_2019}, each task has five classes ($5$-way) and either one or five images per class is used during testing ($1$-shot or $5$-shots).
During training, the number of data points per class is usually set to an arbitrary value, and it remains unclear how this number should be set to achieve the best testing performance.
We focus on training, rather than testing data, because the former can be optimized by following specific procedures for data partitioning and collection.

Intuitively, one would think that the performance always improves with the number of training data points.
However, if the total number of labels is limited, is it better to have a large number of tasks with little data in each task, or a smaller number of highly labelled tasks?
Should some tasks be given more labels than other tasks?
The answers to these questions remain unknown, although they are important to inform the design of new meta-learning benchmarks and the application of meta-learning algorithms to real
problems, especially given that data labelling is costly.
Hence, we address these questions for the first time, for a specific meta-learning algorithm: MAML \cite{finn_model-agnostic_2017}.
Our contributions are:

\begin{itemize}
    \item We introduce the problem of optimizing data allocation in meta-learning, with a fixed budget of total data points to distribute across training tasks. We show that, when tasks are homogeneous, the optimal solution is distributing data uniformly across tasks: all tasks get the same amount of data. This setting is considered in most meta-learning problems (See \emph{'The data allocation problem'} section \ref{section:data_allocation}, Theorem \ref{unifThm}). 
    
    \item When data is distributed uniformly across tasks, we show that the trade-off between number of tasks and number of data points per task, at fixed budget, has a unique solution for the optimum for large budgets (section \ref{sec:uniform}\emph{'Solution of the uniform allocation'}, Theorems \ref{underThm}, \ref{optallocThm}, Figures \ref{fig:linreg}, \ref{fig:image_classification}).
    
    \item Next, we consider the problem of two sets of tasks, easy and hard. When trained separately, we show that hard tasks need more data (per task) than easy tasks. While it is intuitive that hard tasks require more data for training, we emphasize that the total number of data points is fixed by the given budget, therefore the number of tasks is smaller (section \ref{sec:easyhardonly}\emph{'Separate training'}, Figure \ref{fig:easyhardonly}).
    
    \item Finally, we study the problem of training a non-homogeneous mixture of easy and hard tasks. In contrast to when they are trained separately, we show that better performance is obtained by allocating more data to easy tasks. Our interpretation is that, as long as learning transfers from easy to hard tasks, it is better to train more on the former since they are easier to learn. Interestingly, human visual skills also transfer better from easy tasks \cite{ahissar_task_1997} (section \emph{'Joint training'} \ref{sec:easyhard}, Figure \ref{fig:easyhard}).
    
\end{itemize}

We prove results mathematically on mixed linear regression, and confirm those results empirically on few-shot image classification on CIFAR-FS and \emph{mini}-ImageNet (code in the supplementary material).

\section{Related Work}

In the context of meta-learning and mixed linear regression, the work of \cite{kong_meta-learning_2020} asks whether more tasks with a small amount of data can compensate for a lack of tasks with big data.
However, they do not address the problem of finding the optimal allocation of data for a fixed budget, which is the main scope of our work.
The work of \cite{shekhar_adaptive_2020} studies the problem of allocating a fixed budget of data points to a finite set of discrete distributions.
In contrast to our work, they do not study the meta-learning problem and their data has no labels.
Similar to us, a few theoretical studies looked at the problem of mixed linear regression in the context of meta-learning (\cite{bernacchia_meta-learning_2021}, \cite{denevi_learning_2018}, \cite{bai_how_2021}, \cite{tripuraneni_provable_2020}, \cite{du_few-shot_2020}, \cite{collins_why_2020}, \cite{gao_modeling_2020}).
However, none of these studies look into the problem of data allocation, which is our main focus.

An alternative approach to avoid labelling a large amount of data is \emph{active learning}, where a model learns with fewer labels by accurately selecting which data to learn from \cite{settles_active_2010}.
In the context of meta-learning, the option of implementing active learning has been considered in a few recent studies \cite{bachman_learning_2017}, \cite{garcia_few-shot_2018}, \cite{kim_bayesian_2018}, \cite{finn_probabilistic_2019}, \cite{requeima_fast_2020}.
However, they considered the active labeling of data within a given task, for the purpose of improving performance in that task only.
Instead, we ask how data should be distributed across tasks. 

In the context of recommender systems and text classification, a few studies considered whether labeling a data point, within a given task, may increase performance not only in that task but also in all other tasks.
This problem has been referred to as \emph{multi-task active learning} \cite{reichart_multi-task_2008}, \cite{zhang_multi-task_2010}, \cite{saha_online_2011}, \cite{harpale_multi-task_2012}, \cite{fang_active_2017}, or \emph{multi-domain active learning} \cite{li_multi-domain_2012}, \cite{zhang_multi-domain_2016}.
However, none of these studies consider the problem of meta-learning with a fixed budget.
A few studies have looked into actively choosing the next task in a sequence of tasks \cite{ruvolo_active_2013}, \cite{pentina_curriculum_2015}, \cite{pentina_multi-task_2017}, \cite{sun_active_2018}, but they do not look at how to distribute data across tasks.

\section{Meta-learning}\label{sec:meta}

The reader may refer to \cite{hospedales_meta-learning_2020} for a general introduction to meta-learning with neural networks.
In this work, we consider the cross-task setting, where we have a distribution of tasks $\tau\sim p(\tau)$ and a distribution of data points for a given task $\mathcal D^{\tau}\sim p( \mathcal D |\tau)$.
Each task has a loss function $\mathcal{L}(\theta;\mathcal{D})$ that depends on a set of parameters $\theta$ and data $\mathcal{D}$.
Here we assume that the loss has the same functional form across tasks (e.g. square loss if they are all regression tasks, cross-entropy if they are all classification tasks).
The goal of meta-learning is minimizing the mean of the loss across tasks and data.

In the \emph{meta-training} phase, $m$ tasks $(\tau_i)_{i=1}^m$ are sampled from $p(\tau)$ and, for each task, $n_i^t$ training data points $\mathcal D_i^t=(\mathbf{x}_{ij}^t,y_{ij}^t)_{j=1}^{n_i^t}$ and $n_i^v$ validation data points $\mathcal D_i^v=(\mathbf{x}_{ij}^v,y_{ij}^v)_{j=1}^{n_i^v}$, 
are sampled independently from the same distribution $p( \mathcal D |\tau_i)$.
We assume that the data is given by input $\mathbf{x}$ - label $y$ pairs.
The meta-training loss is a function of the data and the meta-parameters $\boldsymbol\omega$, is equal to
\begin{equation}\label{meta-loss-empirical}
\mathcal{L}^{meta}\left(\boldsymbol\omega;\mathcal D^t,\mathcal D^v\right)=\frac{1}{m}\sum_{i=1}^m\frac{1}{n_i^v} \sum_{j=1}^{n_i^v}\mathcal{L}\Big(\boldsymbol\theta(\boldsymbol\omega,\mathcal{D}_i^t); \mathbf{x}_{ij}^v,y_{ij}^v\Big)
\end{equation}
The parameters are adapted to each task $i$ by using the transformation $\theta(\boldsymbol\omega,\mathcal{D}_i^t)$. 
Different meta-learning algorithms correspond to a different choice of this transformation.
Here we use MAML \cite{finn_model-agnostic_2017}, which performs a fixed number of stochastic gradient descent steps with respect to the data for each task.
With a single gradient step, that is equal to
\begin{equation}
\label{adaptation}
\theta(\boldsymbol\omega,\mathcal D_i^t)=\boldsymbol\omega-\frac{\alpha_i}{n_i^t}\sum_{j=1}^{n_i^t}\nabla_{\boldsymbol\omega} \mathcal{L}\left(\boldsymbol\omega;  \mathbf{x}_{ij}^t,y_{ij}^t\right)
\end{equation}
where $\alpha_i$ is the learning rate for task $i$.
This equation corresponds to a full-batch update, employing all the data for a given task, but mini-batch gradient updates can be performed as well.
A number $k$ of gradient steps may be used instead of one.
This step is referred to as \emph{inner loop} of meta-learning.

The loss in Eq.(\ref{meta-loss-empirical}) is minimized with respect to the meta-parameters $\boldsymbol\omega$, namely
\begin{equation}
\label{omegastarmain}
\boldsymbol\omega^\star\left(\mathcal D^t,\mathcal D^v\right)=\mathop{\argmin}_{\boldsymbol\omega}\mathcal{L}^{meta}\left(\boldsymbol\omega;\mathcal D^t,\mathcal D^v\right)
\end{equation}
This minimum is searched by stochastic gradient descent, using a distinct learning rate $\alpha_{meta}$.
At each gradient step, Eq.(\ref{adaptation}) is computed for each task and the gradient of Eq.(\ref{meta-loss-empirical}) with respect to $\boldsymbol\omega$ is taken.
This step is referred to as \emph{outer loop} of meta-learning.
Note that Eq.(\ref{meta-loss-empirical}) includes all $m$ tasks, which translates into full-batch training when taking the gradient.
However, a mini-batch of tasks may be also drawn from the set of $m$ tasks at each step of the optimization.
Standard optimization procedures such as early stopping and scheduling of the learning rate $\alpha_{meta}$ can be applied.
In the case of mixed linear regression (section \emph{'Solution of the uniform allocation'} \ref{sec:uniform}), we solve Eq.(\ref{omegastarmain}) exactly by linear algebra.

In the \emph{meta-testing} phase, the test loss $\mathcal{L}^{test}$ is computed using the optimal value $\boldsymbol\omega^\star$ and test datasets $\tilde{\mathcal D}^t,\tilde{\mathcal D}^v$
\begin{equation}
\label{test_loss}
\mathcal{L}^{test}\left(\mathcal D^t,\mathcal D^v,\tilde{\mathcal D}^t,\tilde{\mathcal D}^v\right)=\mathcal{L}^{meta}\left(\boldsymbol\omega^\star\left(\mathcal D^t,\mathcal D^v\right);\tilde{\mathcal D}^t,\tilde{\mathcal D}^v\right)
\end{equation}
The test datasets correspond to a new draw of both tasks and data points. 
The values of hyperparameters $m, n^t, n^v, \alpha, k$ for meta-testing are not necessarily the same as those used during meta-training.
The main focus of this work is optimizing $m, n_i^t, n_i^v$ for meta-training, while they are fixed during meta-testing. 
To evaluate the performance of the model for a given choice of the hyperparameters, we compute the average test loss, defined as
\begin{equation}
\label{avg_test_loss}
\overline{\mathcal{L}}^{test}(n_1^t, \ldots n_{m}^t, n_1^v,\ldots n_{m}^v) =\mathop{\mathbb{E}}_{\mathcal D_t}\mathop{\mathbb{E}}_{\mathcal D_v}\mathop{\mathbb{E}}_{\tilde{\mathcal D}_t}\mathop{\mathbb{E}}_{\tilde{\mathcal D}_v}\mathcal{L}^{test}
\end{equation}

\section{The data allocation problem}\label{section:data_allocation}

We denote the number of data points per task $i$ during meta-training as $N_i=n_i^t+n_i^v$, equal to the sum of training and validation data.
In all experiments we used an equal split of training and validation, $n_i^t=n_i^v=n_i$.
We assume that the total number of data points for meta-training, referred to as \emph{budget}, is constant and equal to $b=\sum_{i=1}^{m}N_i=2\sum_{i=1}^{m}n_i$.
This is equal to the total number of data points across all training tasks, and is assumed fixed, while the number of data points per task $N_i$ are allowed to vary.
We denote by $\mathbf{n}$ the vector of $n_i$ values, $\mathbf{n}=(n_1,\ldots,n_m)$, and define the \emph{data allocation} problem of finding the value of $\mathbf{n}$ such that the average test loss is minimized
\begin{equation}
\label{dataalloc}
\mathbf{n}^\star=\mathop{\argmin}_{\mathbf{n}\;:\;\sum_{i=1}^{m}n_i=b/2}\overline{\mathcal{L}}^{test}(n_1,\ldots,n_m)
\end{equation}
The optimal value $\mathbf{n}^\star$ is referred to as \emph{optimal allocation}, it may depend on the budget and on other hyperparameters of the model.
The optimal allocation determines which tasks should get more or less data, for a fixed budget $b$ and number of tasks $m$.
In the following theorem, we provide conditions under which the optimal data allocation is uniform.


\vspace{\baselineskip}

\begin{theorem}
\label{unifThm}
If the test loss $\overline{\mathcal{L}}^{test}$ is invariant under permutations of task allocations, i.e. permutations of its arguments $(n_1,\ldots,n_m)$ then the uniform allocation $\mathbf{n}=(n,\ldots,n)$ with $n=\frac{b}{2m}$ is a local extremum of the constrained optimization problem, provided that it is non-degenerate. 

Furthermore, if

\begin{equation}\label{mini-schur}
    \overline{\mathcal{L}}^{test}\left(\frac {n_1 + n_2}{2}, \frac{n_1 + n_2}{2}, n_3,...,n_m\right)\leq{\mathcal{L}}^{test}(n_1,...,n_m), 
\end{equation}
    for all $n_1,...,n_m$, subject to $\sum_{i=1}^k n_i = \frac {b}{2}$,
then the uniform allocation is the global minimum of the data allocation problem.
\end{theorem}

\begin{proof}
The proof of the first part (see Modern Purkiss principle) is given by \cite{waterhouse_symmetric_1983}, noting that the action of the symmetric group preserves the constraint and is irreducible, while the proof of the second part (global minimum) is given by \cite{keilson_global_1967}.
\end{proof}

Note that convexity of the test loss is a sufficient condition for the global minimum. 
We show in the \emph{'Mixed linear regression'} section \ref{sec:linear} that the Purkiss principle applies to the case of mixed linear regression with homogeneous tasks.
This result motivates, in addition to the data allocation problem (\ref{dataalloc}), the study of the \emph{uniform allocation} problem, in which the number of data points is assumed to be equal across tasks, but now the number of tasks $m$ is allowed to vary.
The solution of this problem is defined by
\begin{equation}
\label{unifdataalloc}
n^\star=\mathop{\argmin}_{n\;:\;nm=b/2}\overline{\mathcal{L}}^{test}(n)
\end{equation}
In this case, the question is whether to have more data and less tasks, or less data and more tasks, for the fixed budget $b$.
In the next sections, we study both problems of data allocation and uniform allocation on mixed linear regression and few-shot image classification on CIFAR-FS and \emph{mini}-ImageNet.

\vspace{-1mm}

\subsection{Computation of the optimum}
\label{sec:compopt}

In the case of linear regression, we derive exact expressions for $\overline{\mathcal{L}}^{test}$ and $\mathbf{n}^\star$ in some limiting cases.
In few-shot image classification, and in further linear regression experiments, we estimate $\overline{\mathcal{L}}^{test}$ empirically by searching a grid of values of $\mathbf{n}$.
We average the test loss over multiple repetitions with different data samples and different initial conditions for $\boldsymbol\omega$.
Then, we determine the mean and standard deviation for the optimum $\mathbf{n}^\star$ by the following procedure: we generate multiple instances of test loss/accuracy vs $\mathbf{n}$ by sampling uniformly from the repetitions at each value of $\mathbf{n}$, we record the optimal $\mathbf{n}^\star$ of each instance and construct a distribution of $\mathbf{n}^\star$ across all instances.
We also provide nonlinear (sinusoid) regression experiments in the appendix.

\section{Solution of the uniform allocation}
\label{sec:uniform}
In this section we consider the problem of uniform allocation, while the non-uniform case is studied in the section \emph{'Easy vs hard tasks'} \ref{sec:easyvshard}.
We look at the trade-off between having either more tasks or more data per task, for a fixed budget, and we show that this problem has a unique optimum.
We study this trade-off on two problems: mixed linear regression, where we compute a closed form expression for the optimum, and few-shot image classification, where we show empirical results.

\subsection{Mixed linear regression}
\label{sec:linear}

In mixed linear regression, each task is characterized by a different linear function, and the loss is the mean squared error:

\begin{figure*}[th]
    \centering
    \includegraphics[width = 0.75\textwidth]{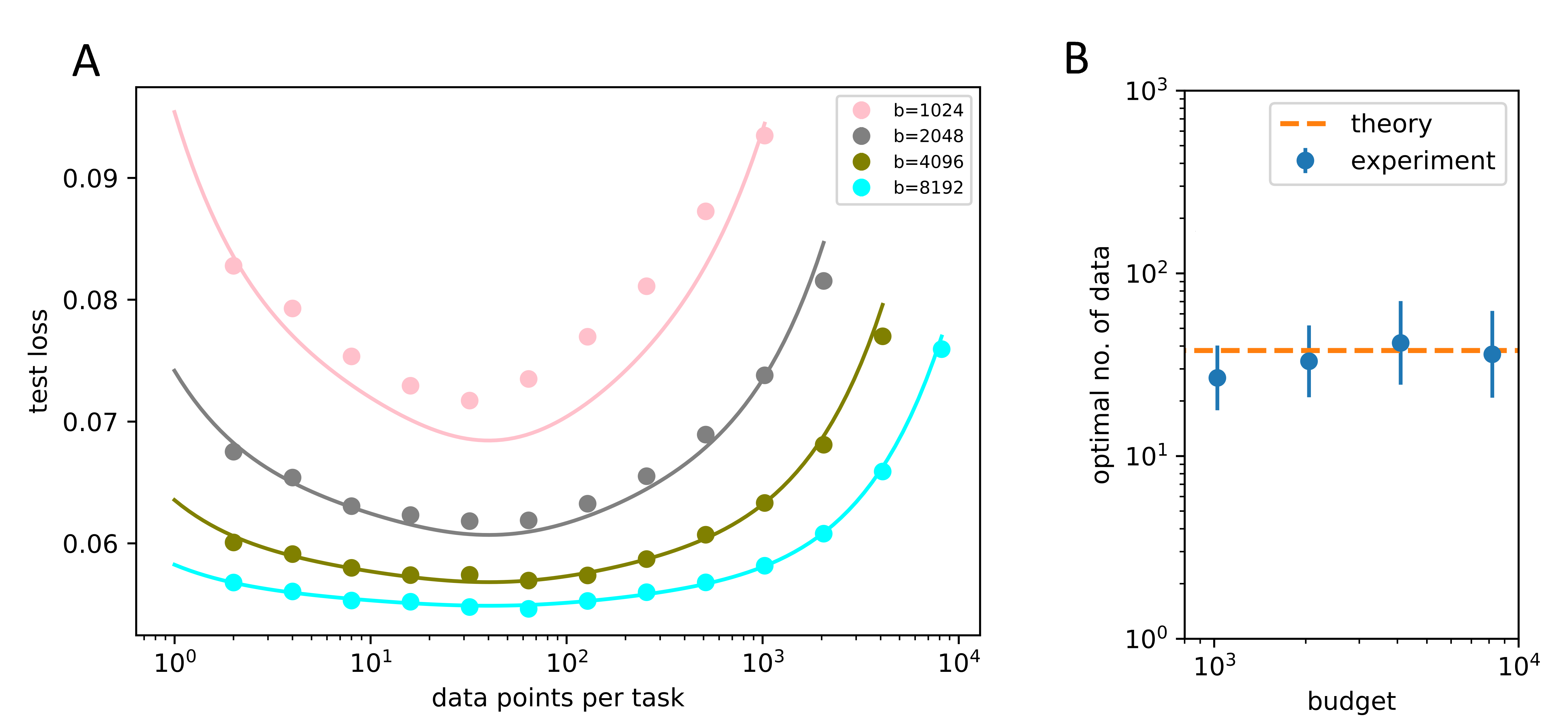}
    \caption{\textbf{The optimal number of data points per task is constant for large budgets: linear regression}. \textbf{A}: Test loss vs. number of data points per task at fixed budget (more data points imply less tasks). Dots: experimental values; Lines: theoretical prediction Eq.(\ref{testloss}), different lines correspond to different budgets (legend). As predicted by Theorem \ref{underThm}, theoretical prediction is more accurate for larger budgets. Each curve has a unique optimum. \textbf{B}: Optimal number of data points per task vs. budget, the four points correspond to the four curves in panel A. The theoretical prediction of Eq.(\ref{optimalN}) (orange line) is close to the estimated experimental optimum (see section \ref{sec:compopt}\emph{'Computation of the optimum'} for its computation).}
    \label{fig:linreg}
\end{figure*}

\begin{align}
\mathcal{L}\left(\boldsymbol\theta; \mathbf{x},y\right)=\frac{1}{2}\left(y-\boldsymbol\theta^T\mathbf{x}\right)^2
\end{align}
where the label $y$ is a scalar, while the input $\mathbf{x}$ and the parameter $\boldsymbol\theta$ are vectors of $p$ components.
Each task corresponds to a different value of the generating parameter $\boldsymbol\theta$.
Across tasks, that is distributed according to a Gaussian
\begin{align}
\label{generative_task}
\boldsymbol\theta\sim\mathcal{N}\left(\boldsymbol\theta_0,\frac{\nu^2}{p}I_p\right)
\end{align}
where $\boldsymbol\theta_0$, $\nu$ are hyperparameters, and $I_p$ is the $p\times p$ identity matrix.
The distribution of data for a given task is given by
\begin{align}
\label{generative_data}
&y\;|\;\mathbf{x},\boldsymbol\theta\sim\mathcal{N}\left(\boldsymbol\theta^T\mathbf{x},\sigma^2\right)\\
&\mathbf{x} \sim \mathcal{N}\left(\mathbf{0},\lambda^2I_p\right)
\end{align}
where $\sigma$ is the label noise and $\lambda$ is the input variability.
Each data point is independently drawn from this distribution, for either training or validation set.
We distinguish between the case of \emph{homogeneous} tasks, where all tasks have the same values of $(\sigma,\lambda)$, and \emph{non-homogeneous} tasks, where we allow those values to vary across tasks.
In the following theorem, we compute an approximate expression for the average test loss for mixed linear regression.

\vspace{\baselineskip}

\begin{theorem} 
\label{underThm}
    Consider the algorithm of section \ref{sec:meta} (MAML one-step) and data generated according to the mixed linear regression model. Let $\sum_{i=1}^mn_i>p$ (underparameterized model), and let $n_i=n_i(\xi)$, $m=m(\xi)$ be any functions of order $\Theta(\xi)$ as $\xi\rightarrow\infty$. Then, the average test loss is equal to
\begin{align}
\label{testloss}
&\overline{\mathcal{L}}^{test}=\frac{\sigma_r^2}{2}\left(1+\frac{\lambda_r^4\alpha_r^2p}{n_r}\right)+\frac{\lambda_r^2h_r\nu^2}{2}+\nonumber\\
&+\frac{\lambda_r^2h_rp}{2}\left[\sum_{i=1}^m\lambda_i^2h_i\right]^{-2}\sum_{i=1}^m\frac{\lambda_i^2}{n_i}\Bigg\{\nonumber\\
&\sigma_i^2\left[h_i+\frac{\lambda_i^4\alpha_i^2}{n_i}\left[\left(n_i+1\right)g_{1i}+p\;g_{2i}\right]\right]+\nonumber\\
&+\frac{\nu^2}{p}\lambda_i^2\left[\left(n_i+1\right)g_{3i}+p\;g_{4i}\right]\Bigg\}+O\left(\xi^{-3}\right)
\end{align}
where the subscript $i$ denotes meta-training hyperparameters for task $i$, while the subscript $r$ denotes meta-testing hyperparameters.
We have defined the function $h_i=\left(1-\lambda_i^2\alpha_i\right)^2+\lambda_i^4\alpha_i^2\frac{p+1}{n_i}$, and the functions $g$ are polynomials in $\lambda_i^2\alpha_i$ with coefficients of order $O(1)$, defined in the appendix, Equation (\ref{gpoly1}).
\end{theorem}

\begin{proof}
The proof is given in the appendix. It provides a generalization of the results of \cite{bernacchia_meta-learning_2021} in the case of non-homogeneous tasks and parametric input variability.
\end{proof}

When tasks are homogeneous ($\sigma_i=\sigma$, $\lambda_i=\lambda$) and a fixed learning rate is used for all meta-training tasks ($\alpha_i=\alpha$), we note that the test loss (\ref{testloss}) is permutation invariant, thus the Purkiss principle of Theorem \ref{unifThm} applies. 
Therefore, in the remainder of this section we consider only the case of uniform allocation ($n_i=n$).
Non-homogeneous tasks and non-uniform allocation are studied in the \emph{'Easy vs hard tasks'} section \ref{sec:easyvshard}. 
Note also that Theorem \ref{underThm} assumes an underparameterized model ($p<\sum_{i=1}^mn_i$).
For completeness, we also study the overparameterized case in the appendix.

\begin{figure*}[ht]
    \centering
    \includegraphics[width = 0.9\textwidth]{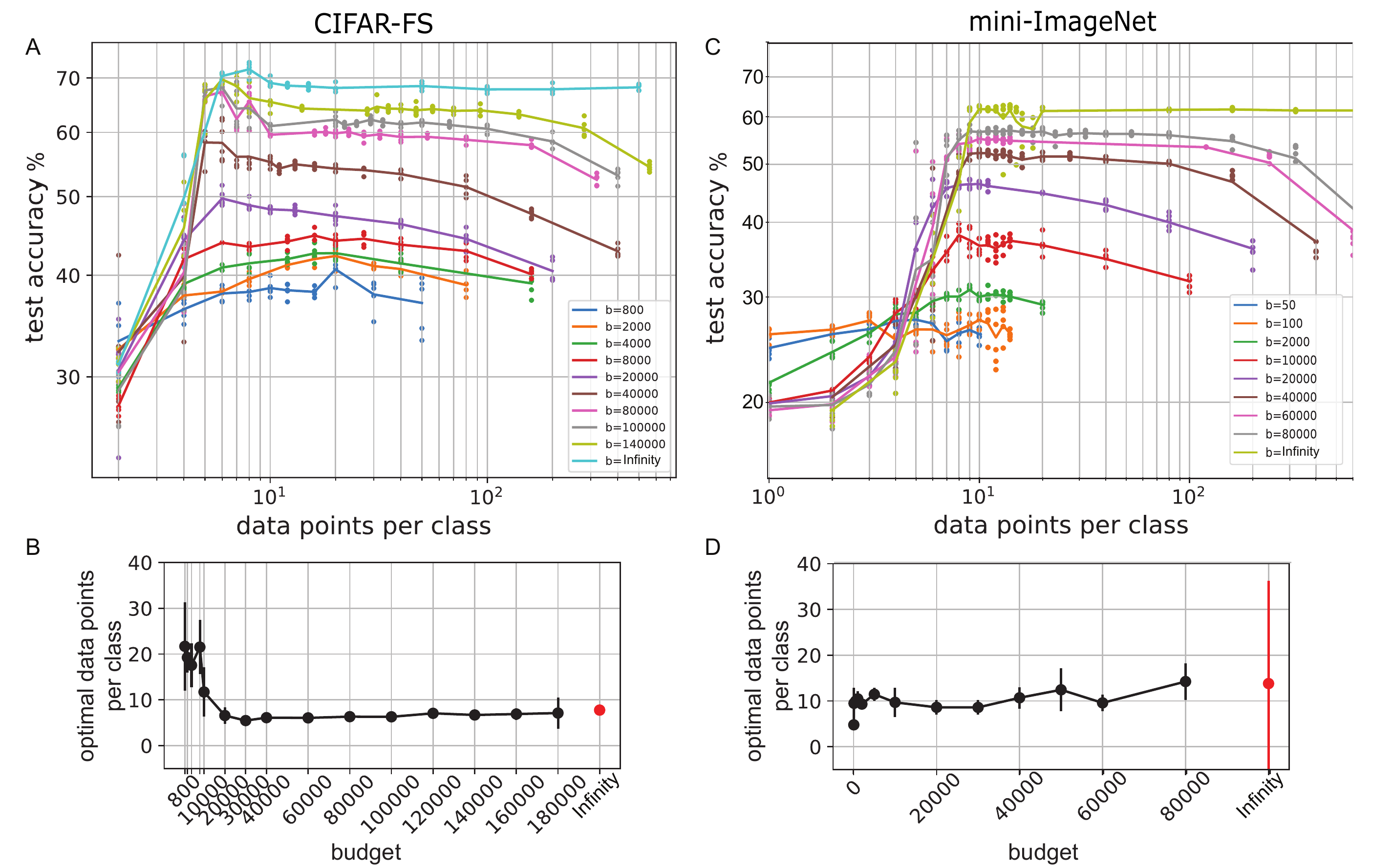}
    \caption{\textbf{The optimal number of data points per task is constant for large budgets: Few-shot image classification}: \textbf{A,B}: CIFAR-FS, \textbf{D,E}: \emph{mini}-ImageNet dataset. Format is the same as of Figure \ref{fig:linreg}. Curves are noisy and tend to flatten at large budgets, but there seems to be a unique optimum for each budget value. The optimum is computed empirically as explained in the \emph{'Computation of the optimum'} section \ref{sec:compopt}. The optimal number of data points converges to $\sim 7$ for CIFAR-FS and to $\sim 10$ for \emph{mini}-ImageNet. Error bars show standard deviation.}
    \label{fig:image_classification}
\end{figure*}

Figure \ref{fig:linreg}A plots the meta-test loss of mixed linear regression as a function of $n$ for different budgets.
It shows a good agreement between the experiments and the theoretical prediction of equation (\ref{testloss}) (see appendix for details).
According to equation (\ref{testloss}), the error between theory and experiment is expected to be of order $O\left(b^{-3/2}\right)$, since $b\sim O\left(\xi^2\right)$, indeed theoretical prediction is more accurate for larger budgets.
As expected, test loss decreases with budget, since more data implies better performance.
We emphasize that curves have a convex shape, implying that there is a unique optimal value of $n$ for each budget.
While the curves tend to flatten at large budgets, the optimum remains approximately constant, as shown in Figure \ref{fig:linreg}B.
In the following theorem, we compute the unique solution of the uniform allocation problem for mixed linear regression.

\vspace{\baselineskip}

\begin{theorem} 
\label{optallocThm}
    Under the assumptions of Theorem \ref{underThm}, consider the test loss of Equation (\ref{testloss}) and the uniform allocation problem in Equation (\ref{unifdataalloc}) . Furthermore, let $p=p(\xi)$ be a function of order $\Theta(\xi)$ as $\xi\rightarrow\infty$, neglect orders $O\left(\xi^{-2}\right)$ in Equation (\ref{testloss}). Then, for all sufficiently small values of the learning rate $\alpha$, the uniform allocation problem has a unique minimum, which does not depend on the budget and is given by 
\begin{equation}
\label{optimalN}
\nonumber
n^\star=Cp\\ 
\end{equation}
where the constant $C$ is defined in Equation (\ref{optimaln}) in the appendix. 

\end{theorem}

\begin{proof}
The proof is provided in the appendix..
\end{proof}

This theorem implies that once the suitable error terms in the approximation of $\mathcal{L}^{test}$ are ignored, there is a unique and constant optimum for the number of data points per task at large budgets. Note that the magnitude of the error terms does depend on the budget and the relation between $n$, $p$ and $m$.
While the theoretical optimum does not depend on the budget, it may depend on whether tasks are hard or easy (see section \ref{sec:easyvshard} \emph{'Easy vs hard tasks'}). 
Figure \ref{fig:linreg}B shows the optimal $n^\star$ as a function of budget,  
it shows that the theoretical value of the optimum (orange line) agrees with the experiments.

\begin{figure*}[ht]
    \centering
    \includegraphics[width = 0.9\textwidth]{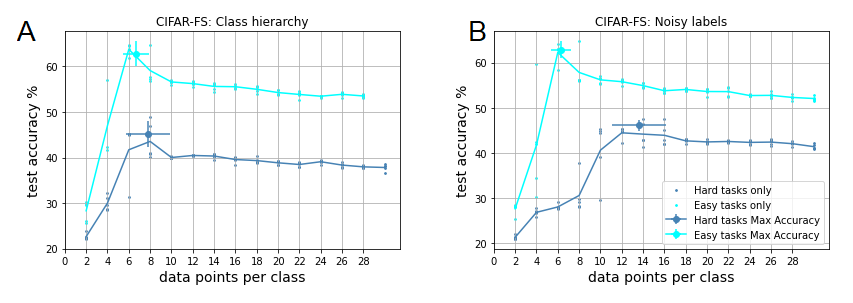}
    \caption{\textbf{Hard tasks prefer more data (and less tasks) when trained separately}. Few shot image classification on CIFAR-FS. \textbf{A} Tasks are made harder by drawing classes within a hierarchy; \textbf{B} Tasks are made harder by adding label noise. Both plots show test accuracy versus the number of data points per class, as in Figure \ref{fig:image_classification}A. Each plot shows an estimate of the point of maximum accuracy, with error bars showing standard deviation (see the  \emph{'Computation of the optimum'} section \ref{sec:compopt} for its computation). In both cases, performance is lower and the optimal number of data points per class is larger for hard tasks.}
    \label{fig:easyhardonly}
\end{figure*}

\subsection{Few-shot image classification}
\label{sec:image}

We next tested whether the results of mixed linear regression generalize to the more interesting problem of few-show image classification.
In this case, the loss function is the cross-entropy, $\mathcal{L}\left(\theta; x,y\right)=-y^T\log\left(f_\theta(x)\right)$,
where $y$ is a one-hot encoding of the class label, and $f_\theta(x)$ is the output vector of a neural network with parameters $\theta$ and input $x$.
We use a convolutional neural network commonly used with MAML on image classification \cite{finn_model-agnostic_2017} (see appendix \ref{appendix:image} for details).

We investigate the CIFAR-FS \cite{bertinetto_meta-learning_2019} and \emph{mini}-ImageNet \cite{vinyals_matching_2017} datasets, which are few-shot versions of CIFAR-100 and ImageNet, respectively.
Both classification problems are $5$-way: each task contains $5$ classes.
We refer to the number of data points \emph{per class}, which has to be multiplied by $5$ to find the number of data points \emph{per task}. 
As in previous studies, we used $5$ \textit{shots} during testing ($5$ data points per class), while the number of shots during training depends on the data allocation. 

In previous work \cite{vinyals_matching_2017}, \cite{bertinetto_meta-learning_2019}, we note that tasks are usually re-sampled indefinitely until convergence of the model, thus there is no limit on the number of tasks that can be generated.
We instead pre-sample a set of tasks in order to fix the budget constraint. 
For comparison, we also run experiments in the usual way, and we call this the \emph{infinite budget} case.
However, the total number of labels is fixed and, even if tasks are re-sampled indefinitely, it does not imply that the amount of data is infinite, rather the same image may appear in multiple tasks. 

As expected, Figure \ref{fig:image_classification} shows that test performance improves with the budget, for both CIFAR-FS and \emph{mini}-ImageNet (Figure \ref{fig:image_classification}A,C). 
For infinite budget, the accuracy is similar to previously reported values ($\sim 63\%$ for \emph{mini}-ImageNet \cite{finn_model-agnostic_2017}, $\sim 71\%$ for CIFAR-FS \cite{bertinetto_meta-learning_2019}).
For CIFAR-FS, the optimal number of data points per class was $\sim 20$ at small budgets, but decreased and remained approximately constant at $\sim 7$ for large budgets (Figure \ref{fig:image_classification}B).
For \emph{mini}-ImageNet, the optimal number of data points per class was $\sim 5$ at very small budget and then increased and remain approximately constant at $\sim 10$ (Figure \ref{fig:image_classification}D). 
The performance curves in Figure \ref{fig:image_classification}A,C tend to flatten at higher budgets, but the optimum does not change significantly.
Overall, the empirical study of both datasets confirms our prediction that the optimal number of data points per task is constant at large budgets.

\begin{figure*}[ht]
    \centering
    \includegraphics[width = 0.9\textwidth]{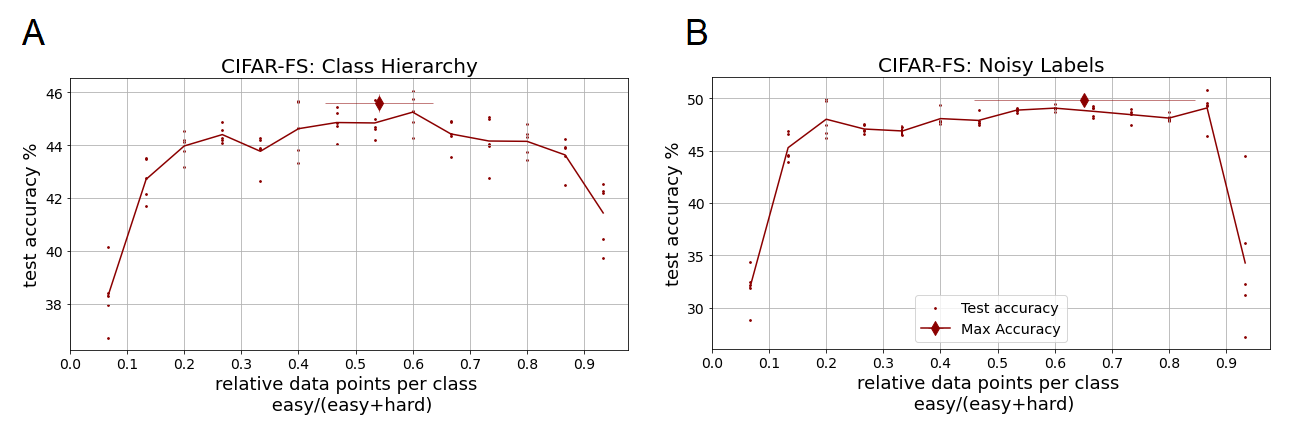}
    \caption{\textbf{When training on a mixture of easy and hard tasks, it is better to allocate more data to easy tasks}. Few shot image classification on CIFAR-FS. \textbf{A} Tasks are made harder by drawing classes within a hierarchy; \textbf{B} Tasks are made harder by adding label noise. Both plots show test accuracy versus the relative amount of easy vs hard data points per class. Each plot shows an estimate of the point of maximum accuracy, with error bars showing standard deviation. In both cases, a slightly higher performance is obtained by allocating more data to easy tasks than to hard ones.}
    \label{fig:easyhard}
\end{figure*}

\section{Easy vs hard tasks}
\label{sec:easyvshard}

In this section we consider the case of non-homogeneous tasks.
We distinguish between two sets of tasks, easy and hard.
We use two independent definitions of hard tasks, one affects the input and the other affects the output (label) of a dataset.
We apply this definition in a similar way to both mixed linear regression and few-shot image classification.

For the problem of mixed linear regression, we define \emph{task difficulty} in terms of the hyperparameters $\sigma$ and $\lambda$. 
A task is harder if it has a larger $\sigma$ (at equal $\lambda$) or smaller $\lambda$ (at equal $\sigma$).
The case of larger $\sigma$ is intuitive, a task is harder to learn if labels are more corrupted by noise.
In the case of smaller $\lambda$, the smaller input range makes it harder to solve the regression problem in presence of noise.

In few-shot image classification, the first method to make a task harder is to introduce label noise \cite{song_learning_2021}: each input image has $20\%$ probability of having its label swapped with another random class.
The second method is similar to \cite{collins_why_2020}: we take advantage of the hierarchical tree of the CIFAR-100 dataset and we constraint each task to draw classes from one of three superclasses: 1) animals, 2) vegetations, 3) object and scenes.
Therefore, we assume that each task has a smaller variability of its input, not in terms of pixels color or intensity, but in terms of semantic relations.
Intuitively, it is harder to distinguish inputs when they are more similar to each other.
We refer to the two different definitions as, respectively, \emph{noisy labels} and \emph{class hierarchy}.

\vspace{-1mm}

\subsection{Separate training}
\label{sec:easyhardonly}

Before studying the training of a mixture of easy and hard tasks, we ask what is the optimal uniform allocation when the two types of tasks are trained separately.
In mixed linear regression, the expression for the optimum of the uniform allocation $n^\star$ is given by Eq.(\ref{optimalN}), but is hard to evaluate how it depends on $\sigma$ and $\lambda$.
Therefore we computed an approximation that holds for small $\alpha'$ (see equation (\ref{nstarapprox}) in the appendix):
\begin{align} \label{Nstar}
n^\star= \left[2\left(1+\frac{\sigma'^2}{\nu^2}\right)\right]^{\frac{1}{3}}\alpha'^{\frac{4}{3}}p+O\left(\alpha'^{\frac{5}{3}}\right)
\end{align}
where $\alpha'=\lambda^2\alpha$ and $\sigma'=\sigma/\lambda$.
The optimum increases with $\sigma$, suggesting that harder tasks require more data (and less tasks) at fixed budget.
For $\lambda$, there are two opposing forces: 1) On one hand a smaller $\lambda$ is equivalent to amplifying output noise $\sigma'$ and increasing the optimum $n^\star$; 2) On the other hand, $\lambda$ rescales the learning rate $\alpha'$ with the opposite and stronger effect that a smaller $\lambda$ decreases the optimum.

Figure \ref{fig:easyhardonly} shows that introducing task difficulty in few-shot image classification on CIFAR-FS increases the optimum for both methods (A: class hierarchy; B: noisy labels, see section \ref{sec:easyhardapp} \emph{'Easy and hard tasks creation in the image classification experiments'} for details).
As expected, performance is lower for hard tasks in both cases (note that we train and test on the same set of tasks, either only easy or only hard).
While it is intuitive that hard tasks require more data to learn, we emphasize that, for a fixed budget, this comes at the expense of a smaller number of tasks.

\vspace{-1mm}

\subsection{Joint training}
\label{sec:easyhard}

We now turn to the problem of training on a mixture of easy and hard tasks.
In addition to a fixed budget, we further assume an equal number of easy and hard tasks, and a constant sum of easy and hard data points per task.
Therefore, the only hyperparameter of interest is the relative number of data points per task for easy vs hard.
Note that we use a mixture of easy and hard tasks also for testing, but we always use an equal number of easy and hard data points and tasks in that case (see section \emph{'Easy and hard tasks creation in the image classification experiments'} \ref{sec:easyhardapp} for details).

After the results of section \ref{sec:easyhardonly}, we expect better results when allocating more data to hard tasks.
Surprisingly we find that the opposite is true.
Figure \ref{fig:easyhard} shows that a slightly higher performance is obtained when allocating more data to easy tasks, in few-shot image classification on CIFAR-FS for both methods (Panel A: class hierarchy; Panel B: noisy labels).
Intuitively, easy tasks are easier to learn than hard tasks.
Therefore, it may be that if training on easy tasks transfers to better performance on hard tasks, then it is better to allocate more data to easy tasks.

\section{Discussion}

In this paper we analysed the problem of optimal data allocation in meta-learning when the budget of labelled examples is limited.
When tasks are homogeneous, we showed that uniform data allocation across tasks is optimal (under the assumptions of Theorem \ref{unifThm}).
We further studied whether one should use less tasks with more data or more tasks and less data.
For mixed linear regression, we found a unique solution for the optimum at large budgets. 
We confirmed this finding empirically on few-shot image classification (an example of nonlinear regression is also included in the appendix).

In the case of non-homogeneous tasks, with a mixture of easy and hard tasks, we showed how to optimally allocate data between the two types of tasks.
In particular, we found that it is better to allocate more data to easy tasks.
This result echoes findings in experimental neuroscience, where it was found that human visual skills indeed transfer better from easy tasks than from hard ones \cite{ahissar_task_1997}.
Our findings provide a guideline for collecting meta-learning data in a way that achieves the best performance under a fixed budget.
We do not expect our study to have a negative societal impact, at least not in a direct way. 

Overall, our study exemplifies the importance of optimal data allocation in meta-learning and gives a series of empirical and theoretical insights on the relation between model performance and data allocation for MAML. 
While the behaviour of other meta-learners need not be the same, we surmise that the problem of training models close to optimal allocation is important, and leave much space for empirical study in a variety of contexts, as well as for the development of a more general theoretical framework.
For example, we have only scratched the surface of the problem of non-uniform allocation, which requires much further study.
\bibliography{mylibrary}

\clearpage

\label{appendix}

\section{Details of experiments}

\subsection{Mixed linear regression experiments}
\label{sec:app_linear}

For the distributions of Eqs.(\ref{generative_task}),(\ref{generative_data}), we used the following values of parameters: $\sigma=0.2$, $\nu=0.2$, $p=128$, $\boldsymbol\theta_0=(0.05,\ldots,0.05)$.
During meta-training, we run experiments for several pairs of values of $(N,b)$, as shown in Figure \ref{fig:linreg}.
As explained in the main text, the number of training and validation data points was equal, $n_t=n_v=n$, and the total number of data points per task was defined as $N=n_t+n_v=2n$.
We used an inner loop learning rate of $\alpha=0.3$ during training.
We computed the exact minimum $\boldsymbol\omega^\star$ by using standard linear algebra, since the loss in Eq.(\ref{meta-loss-empirical}) is convex and quadratic.

During meta-testing, we estimated the test loss by averaging over $100$ tasks, for each task we used $20$ training and $50$ validation data points and we used a inner loop learning rate of $0.3$.
We repeated each experiment $100$ times, with different generated data samples, for each pair of values of $(N,b)$.
As explained in the main text, we used sampling to determine the optimum of the test loss: for each value of the budget $b$, we sampled $1000$ curves for the test loss, by sampling one out of the $100$ repetitions for each value of the number of data points per task $N$.
For each one of the $1000$ curves, we determined the empirical value of the minimum $N^\star$ of the test loss, and we computed the mean and standard deviation of $N^\star$ across those $1000$ values (shown in Figure \ref{fig:linreg}B,E).

In addition to the empirical experiments on mixed linear regression, that are conducted by sampling the generative model of Eqs.(\ref{generative_task}),(\ref{generative_data}), we used the analytical formulas for the average test loss and the optimal data allocation computed in section \emph{'Computation of the average test loss'}\ref{sec:loss}
and  \emph{'Optimal uniform allocation in mixed linear regression'}\ref{sec:optuniall}, respectively.

\subsection{Image classification experiments} 
\label{appendix:image}
For the CIFAR-FS and \emph{mini}-ImageNet we apply MAML on a base learner given by a convolutional neural network (CNN) with architecture as described in \cite{finn_model-agnostic_2017}. Small modifications were made to run on the CIFAR-FS data. We use a network with 4 convolutional blocks. Each block consisted of a sequence of 2D convolutions with kernel size 3, stride 1, same padding, batch normalization, a ReLU nonlinearity, and MaxPooling to half size with the kernel size and the stride of 2. The predictor head is a softmax layer applied to the flattened resulting features. We used 64 filters per layer for CIFAR-FS and 32 for \emph{mini}-ImageNet, in order to reduce overfitting, as in \cite{finn_model-agnostic_2017}. CIFAR-FS and \emph{mini}-ImageNet image sizes were 32 $\times$ 32 and 84 $\times$ 84, respectively.

For each budget limited run of MAML, data was presampled and arranged into the required number of tasks and data points per class. Whether from the training, testing or validation meta-datasets, a task was sampled by selecting 5 random classes with replacement from the available pool. Thus, it is possible for independent samples of tasks to occasionally contain the same class of images.

For the grid search algorithm, independent runs were performed by independently sampling the initial conditions of the algorithm and also independently sampling the meta-training, meta-validation and meta-test datasets. We performed 5 independent runs this way for each point on the grid. For the  meta-test datasets, 1000 tasks were sampled to minimize variance in the estimates. Each task consists of 5 randomly sampled classes (5-way), 5 data points (5-shot) per class for model adaptation and 5 data points per class for model evaluation.


During meta-training, the train-test split for each meta-training task dataset was 0.5. We ran the Adam Algorithm in the meta-update of the MAML parameter $\omega$ in the outer loop with the initial learning rate of 0.001. We annealed this learning rate on a plateau of the validation loss for at least 250 meta-updates. This is done four times with the fifth incurring termination of training.
For the adaptation step we ran 5 steps of gradient descent during meta-training and 1 step during meta-testing. Our preliminary experimental results demonstrate that using 5 adaptation steps for meta-training gives more stable performance than using only 1 adaptation step although no performance improvement can be obtained by using more adaptation steps. 1 adaptation step was used during training for \emph{mini}-Imagenet experiments to accommodate the larger memory requirements associated with this dataset. The inner loop learning (i.e. adaptation step size) is set to 0.01 in all experiments. 

The mini-batch gradient descent is used during meta-training for efficient GPU memory use. Each mini-batch consists of 25 tasks for the experiments with $m\ge 25$ and otherwise $m$ tasks, and each task consists of 50 data points for experiments with $N\ge 50$ and otherwise N data points. We performed preliminary experiments and no significant difference in terms of performance was observed when increasing the batch size further.

All experiments ran on a single Nvidia 2080 GPU with an average runtime of 1.5 hours per run of the CIFAR-FS dataset and 5 hours for miniImageNet. These times depend on the exact hyperparameters such as the allocation and budget. Overall, grid search experiments were distributed in parallel over 15 - 45 GPUs.

\subsection{Easy and hard tasks creation in the image classification experiments}
\label{sec:easyhardapp}
In the \emph{'Easy vs hard tasks'} section \ref{sec:easyvshard} we introduce hard tasks to test the influence of task difficulty and study the optimal allocation when training only in easy tasks, only in hard tasks a mixture of hard and easy tasks. 
The experimental set-up follows the description in the appendix. \ref{appendix:image}.
We define easy task as the standard tasks of CIFAR-FS, while we use two definitions of hard task:
\begin{itemize}
  \item \emph{Class hierarchy}. Task difficulty is represented by the similarity in the input images.
  The CIFAR-FS dataset can be described by a 4-level hierarchy. 
  At the top level, there are three high-level categories: \emph{animals}, \emph{vegetations}, and \emph{objects and scenes} (see Table \ref{table:hierarchy}).   
  In a hard task all classes belong to the same high-level category. 
 \item \emph{Noisy labels} Task difficulty is represented by label noise \cite{song_learning_2021}. 
 In hard tasks, each image has a $20\%$ probability of being mislabeled, by randomly drawing a label from another class with a uniform distribution.
\end{itemize}
The train, test and validation split in hard and easy tasks is the original CIFAR-FS split.
In the experiments with only hard tasks, the procedure is exactly the same as described in section \ref{appendix:image} but only employing hard tasks for both training and testing.
In joint training with easy and hard tasks, we keep the budget constant at $50,000$ data points. 
The number of easy tasks is also kept constant, at $333$, and equal to the number of hard tasks.
The sum of easy and hard data points per class is kept constant and equal to $30$ data points per class.
The balance of easy and hard tasks varies between $2$ and $28$ data points per class.

\begin{table*}[ht]
  \caption{CIFAR-FS, 4-level Hierarchy}
  \label{table:hierarchy}
  \centering
  \begin{tabular}{llll}
    \toprule
    3 Class  & 10 Class & 20 Class & 100 Class\\
    \midrule
  animals &large animals  & reptiles & crocodile, dinosaur, lizard, snake, turtle\\
    & &large carnivores & bear, leopard, lion, tiger, wolf\\
    & &large omnivores and herbivores &camel, cattle, chimpanzee, elephant, kangaroo\\
    \cmidrule(l{10em}r{0em}){1-4}
    &medium animals &aquatic mammals & beaver, dolphin, otter, seal, whale\\
    & &medium-sized mammals& fox, porcupine, possum, raccoon, skunk\\
    \cmidrule(l{10em}r{0em}){1-4}
    &small animals&small mammals& hamster, mouse, rabbit, shrew, squirrel\\
    &&fish& aquarium fish, flatfish, ray, shark, trout\\
    \cmidrule(l{10em}r{0em}){1-4}
    & invertebrates & insects & bee, beetle, butterfly, caterpillar, cockroach\\
    &&non-insect& crab, lobster, snail, spider, worm\\
    \cmidrule(l{10em}r{0em}){1-4}
    &people &people & baby, boy, girl, man, woman\\
    \midrule
    vegetations &vegetations& flowers&orchids, poppies, roses, sunflowers, tulips\\
    &&fruit and vegetables& apples, mushrooms, oranges, pears, peppers\\
    &&trees& maple, oak, palm, pine, willow\\
    \midrule
    objects and scenes&household objects & food containers &bottles, bowls, cans, cups, plates\\
    &&household electrical devices& clock, keyboard, lamp, telephone, television\\
    &&household furniture &bed, chair, couch, table, wardrobe\\
    \cmidrule(l{10em}r{0em}){1-4}
    &construction& large man-made outdoor things& bridge, castle, house, road, skyscraper\\
    \cmidrule(l{10em}r{0em}){1-4}
    &natural scenes& large natural outdoor scenes &cloud, forest, mountain, plain, sea\\
    \cmidrule(l{10em}r{0em}){1-4}
    &vehicles&vehicles 1 &bicycle, bus, motorcycle, pickup truck, train\\
    &&vehicles 2& lawn-mower, rocket, streetcar, tank, tractor\\
    \bottomrule
  \end{tabular}
\end{table*}

\section{Nonlinear regression}
\label{sec:sinusoid}
We investigated non-linear regression on a 1-dimensional sinusoid wave dataset. 
The loss is again the mean squared error:
\begin{align}
\mathcal{L}\left(\theta; x,y\right)=\frac{1}{2}\left(y-f_\theta(x)\right)^2
\end{align}
where $f_\theta(x)$ is the output of a neural network with parameters $\theta$ and input $x$.
We use a simple Multi-Layer Perceptron, following the architecture used in MAML \cite{finn_model-agnostic_2017}, see below for details on this experiment.

The task distribution is a joint over 2 parameters $\tau = (A,\phi)$, where both distributions are uniform within their range, given by
\begin{align}
&A\sim\mathcal{U}\left(0.1,5\right)
&\phi\sim\mathcal{U}\left(0,\pi\right)
\end{align}
The distribution of data for a given task $\tau = (A,\phi)$ is given by
\begin{align}
&x \sim \mathcal{U}\left(-5,5\right)
&y = A \cdot \sin(x+\phi) 
\end{align}
Each data point, of either training or validation sets, is independently drawn from this distribution.
We don't add any label noise, therefore the minimum achievable value of the test loss is zero.

Figure \ref{fig:sinusoid} shows the data allocation results for the sinusoid regression, the format is the same as in Figure \ref{fig:linreg}.
The curves for the test loss are qualitatively similar to one of the two cases analysed for linear regression (Figure \ref{fig:linreg}A,B,C).
Again, performance generally increases with budget and test loss curves tend to have a unique minimum.
Therefore, different data allocations for a given budget can result in significantly different performance and there exists an optimal allocation for a given budget. 
Figure \ref{fig:sinusoid}B shows that the optimum $N^\star$ increases from a small value at small budgets to an approximately constant value for large budgets, of about $150$ data points per task.
Our theoretical results of a constant optimum at large budgets, obtained in section \ref{sec:uniform} for mixed linear regression, seems to be confirmed in the case of nonlinear regression.

The neural network used in all sinusoid experiments was a MLP with 2 hidden layers of 40 nodes each with ReLU nonlinearities. 
During meta-training we used full batch gradient descent. The initial learning rate for performing gradient updates in the outer loop of the MAML algorithm was 0.001, while the learning rate for the adaptation procedure (inner loop of MAML) was 0.01.
In all experiments we used a learning rate annealing schedule upon reaching a plateau in the training loss. This was done three times with the fourth incurring termination of training. Additionally, we only started the learning rate scheduler after a predetermined number of meta-iterations which depended on the budget.

Unlike the experimental settings in other meta-learning papers, we do not use a fixed number of shots for meta-training since we focus on the investigation of different numbers of data points per task. 
The train-test split for each meta-training task dataset was 0.5. As a result, the minimum number of data points per task was 2 (i.e. 1 for training and the other for validation). The maximum number of data points per task was equal to the budget $b$ in which case there was only one task (m=1). The training was done using full batches of data and the Adam optimizer (\cite{kingma_adam:_2014}).

During meta-testing, we randomly sampled 1000 tasks and 500 data points per task for adaptation, and 500 data points per task for computing the test loss. 
We used one step adaptation during meta-training and five steps during meta-testing.

For the grid search, we performed up to 10 independent runs for each point on the grid with randomly sampled meta-training data and randomly initialised model weights. To obtain the statistics of the optimal $N^\star$ for each budget in the grid search experiments, we employed the following strategy. We randomly sampled one run for each point on the grid and derived the optimal $N^\star$ from this sample and repeated the sampling for 1000 times. The mean and standard deviation of $N^\star$ are computed and presented in Figure \ref{fig:sinusoid}B.


\begin{figure}[ht]
    \centering
    \includegraphics[width = 0.5\textwidth]{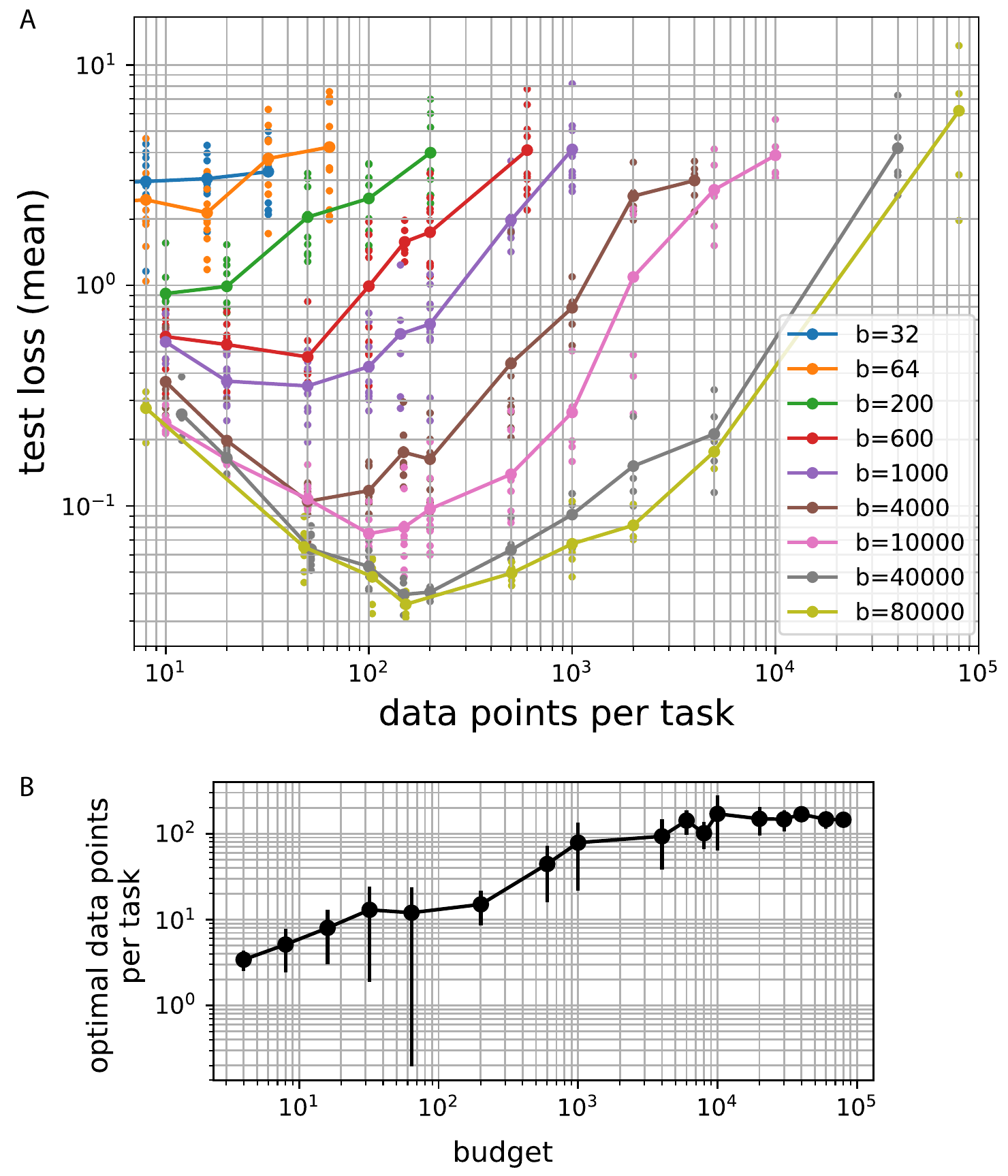}
    \caption{\textbf{Nonlinear regression.} Format is the same as of Figure \ref{fig:linreg}. Results are qualitatively similar to Figure \ref{fig:linreg}A,B with the optimal number of data points per task increasing and then converging to a constant number ($\sim 150$) at large budgets (panel B).}
    \label{fig:sinusoid}
\end{figure}

\section{Computation of the average test loss}
\label{sec:loss}

This section largely follows the results of \cite{bernacchia_meta-learning_2021}, but here the calculations are generalized to the case of heterogeneous tasks and parametric input variability. 

\subsection{Definition of the loss}

We consider the problem of mixed linear regression $\mathbf{y}=X\mathbf{w}+\mathbf{z}$ with squared loss, where $X$ is a $n\times p$ matrix of input data, each row is one of $n$ data vectors of dimension $p$, $\mathbf{z}$ is a $n\times 1$ noise vector, $\mathbf{w}$ is a $p\times 1$ vector of generating parameters and $\mathbf{y}$ is a $n\times 1$ output vector.
Data is collected for $m$ tasks, each with a different value of the parameters $\mathbf{w}$ and a different realization of the input $X$ and noise $\mathbf{z}$.
We denote by $\mathbf{w}^{(i)}$ the parameters for task $i$, for $i=1,\ldots,m$.
For a given task $i$, we denote by $X_i^t, X_i^v$ the input data for, respectively, the training and validation sets, by $\mathbf{z}_i^t, \mathbf{z}_i^v$ the corresponding noise vectors and by $\mathbf{y}_i^t, \mathbf{y}_i^v$ the output vectors.
We denote by $n_i^t$, $n_i^v$ the data sample size for task $i$, respectively for training and validations sets.

For a given task $i$, the training output is equal to 
\begin{equation}
\mathbf{y}_i^t=X_i^t\mathbf{w}^{(i)}+\mathbf{z}_i^t
\end{equation}
Similarly, the validation output is equal to
\begin{equation}
\mathbf{y}_i^v=X_i^v\mathbf{w}^{(i)}+\mathbf{z}_i^v.
\end{equation}

We consider MAML as a model for meta-learning (Finn et al 2017).
The meta-training loss is equal to
\begin{equation}
\label{lossmeta}
\mathcal{L}^{meta}=\frac{1}{m}\sum_{i=1}^m \frac{1}{2n_i^v}\left|\mathbf{y}_i^v-X_i^v\boldsymbol\theta_i(\boldsymbol\omega)\right|^2
\end{equation}
where vertical brackets denote euclidean norm, and the estimated parameters $\boldsymbol\theta_i(\boldsymbol\omega)$ are equal to the one-step gradient update on the single-task training loss $\mathcal{L}^{(i)}= |\mathbf{y}_i^t-X_i^t\boldsymbol\theta_i|^2/2n_i^t$, with initial condition given by the meta-parameter $\boldsymbol\omega$.
The single gradient update is equal to
\begin{equation}
\boldsymbol\theta_i(\boldsymbol\omega)=\left(I_p-\frac{\alpha_i}{n_i^t}{X_i^t}^TX_i^t\right)\boldsymbol\omega+\frac{\alpha_i}{n_i^t}{X_i^t}^T\mathbf{y}_i^t
\end{equation}
where $I_p$ is the $p\times p$ identity matrix and $\alpha_i$ is the learning rate.
We seek to minimize the meta-training loss with respect to the meta-parameter $\boldsymbol\omega$, namely
\begin{equation}
\label{omegastar}
\boldsymbol\omega^\star=\mbox{arg}\min_{\boldsymbol\omega}\mathcal{L}^{meta}
\end{equation}
We evaluate the solution $\boldsymbol\omega^\star$ by calculating the meta-test loss
\begin{equation}
\label{losstest0}
\mathcal{L}^{test}= \frac{1}{2n_s}\left|\mathbf{y}^s-X^s\boldsymbol\theta^\star\right|^2
\end{equation}
Note that the test loss is calculated over test data $X^s, \mathbf{z}^s$, and test parameters $\mathbf{w'}$, namely
\begin{equation}
\label{testoutput}
\mathbf{y}^{s}=X^{s}\mathbf{w'}+\mathbf{z}^{s}
\end{equation}
Furthermore, the estimated parameters $\boldsymbol\theta^\star$ are calculated on a separate set of target data $X^{r}, \mathbf{z}^{r}$, namely
\begin{equation}
\label{thetastar}
\boldsymbol\theta^\star=\left(I_p-\frac{\alpha_r}{n_r}{X^{r}}^TX^{r}\right)\boldsymbol\omega^\star+\frac{\alpha_r}{n_r}{X^{r}}^T\mathbf{y}^{r}
\end{equation}
\begin{equation}
\mathbf{y}^{r}=X^{r}\mathbf{w'}+\mathbf{z}^{r}
\end{equation}
Note that the learning rate and sample size can be different at testing, denoted by $\alpha_r, n_r, n_s$.
We are interested in calculating the average test loss, that is the test loss of Eq.\ref{losstest0} averaged over the entire data distribution, equal to
\begin{equation}
\label{losstest}
\overline{\mathcal{L}}^{test}=   \mathop{\mathbb{E}}_{\mathbf{w}}\mathop{\mathbb{E}}_{\mathbf{z}^{t}}\mathop{\mathbb{E}}_{X^{t}}\mathop{\mathbb{E}}_{\mathbf{z}^{v}}\mathop{\mathbb{E}}_{X^{v}}\mathop{\mathbb{E}}_{\mathbf{w}'}\mathop{\mathbb{E}}_{\mathbf{z}^s}\mathop{\mathbb{E}}_{X^s}\mathop{\mathbb{E}}_{\mathbf{z}^r}\mathop{\mathbb{E}}_{X^r}\frac{1}{2n_s}\left|\mathbf{y}^s-X^s\boldsymbol\theta^\star\right|^2
\end{equation}

\subsection{Probability distributions and averaging}

We assume that all random variables are Gaussian.
In particular, we assume that the rows of the matrix $X$ are independent, and each row, denoted by $\mathbf{x}$, is distributed according to a multivariate Gaussian with zero mean and unit covariance
\begin{equation}
\mathbf{x}\sim\mathcal{N}\left(0,\lambda^2I_p\right)
\end{equation}
where $I_p$ is the $p\times p$ identity matrix.
Similarly, the noise is distributed following a multivariate Gaussian with zero mean and variance equal to $\sigma^2$, namely
\begin{equation}
\mathbf{z}\sim\mathcal{N}\left(0,\sigma^2I_n\right)
\end{equation}
Finally, the generating parameters are also distributed according to a multivariate Gaussian of variance $\nu^2/p$, namely
\begin{equation}
\mathbf{w}\sim\mathcal{N}\left(\mathbf{w}_0,\frac{\nu^2}{p}I_p\right)
\end{equation}
The generating parameter $\mathbf{w}$ is drawn once and kept fixed within a task, and drawn independently for different tasks.
The values of $\mathbf{x}$ and $\mathbf{z}$ are drawn independently in all tasks and datasets (training, validation, target, test).
In order to perform the calculations in the next section, we need the following results.

\begin{lemma} 

\label{lemma1}

Let $X$ be a Gaussian $n \times p$ random matrix with independent rows, and each row has covariance equal to $I_p$, the $p\times p$ identity matrix. Then:
    \begin{align}
    \label{1ord}
    &\mathbb{E}\left[X^TX\right]=\lambda^2nI_p\\
    \label{2ord}
    &\mathbb{E}\left[\left(X^TX\right)^2\right]=\lambda^4n\left(n+p+1\right)I_p=\lambda^4n^2\mu_2I_p\\
    \label{3ord}
    &\mathbb{E}\left[\left(X^TX\right)^3\right]=\lambda^6n\left(n^2+p^2+3np+3n+3p+4\right)I_p=\nonumber\\
    &=\lambda^6n^3\mu_3I_p\\
    \label{4ord}
    &\mathbb{E}\left[\left(X^TX\right)^4\right]=\lambda^8n\left(n^3+p^3+6n^2p+6np^2+\right.\nonumber\\
    &\left.+6n^2+6p^2+17np+21n+21p+20\right)I_p=\lambda^8n^4\mu_4I_p\\
    \label{2ordTr}
    &\mathbb{E}\left[X^TX\;\mbox{Tr}\left(X^TX\right)\right]=\lambda^4\left(n^2p+2n\right)I_p=\lambda^4pn^2\mu_{1,1}I_p\\
    \label{3ordTr}
    &\mathbb{E}\left[\left(X^TX\right)^2\mbox{Tr}\left(X^TX\right)\right]=\nonumber\\
    &=\lambda^6n\left(n^2p+np^2+np+4n+4p+4\right)I_p=\nonumber\\
    &=\lambda^6pn^3\mu_{2,1}I_p\\
    \label{3ordTr2}
    &\mathbb{E}\left[X^TX\mbox{Tr}\left(\left(X^TX\right)^2\right)\right]=\nonumber\\
    &=\lambda^6n\left(n^2p+np^2+np+4n+4p+4\right)I_p=\nonumber\\
    &=\lambda^6pn^3\mu_{1,2}I_p\\
    \label{4ordTr}
    &\mathbb{E}\left[\left(X^TX\right)^2\mbox{Tr}\left(\left(X^TX\right)^2\right)\right]=\nonumber\\
    &=\lambda^8n\left(n^3p+np^3+2n^2p^2+2n^2p+2np^2+\right.\nonumber\\
    &\left.+8n^2+8p^2+21np+20n+20p+20\right)I_p=\nonumber\\
    &=\lambda^8pn^4\mu_{2,2}I_p
    \end{align}
where the last equality in each of these expressions defines the variables $\mu$.
Furthermore, for any $n\times n$ symmetric matrix C and any $p\times p$ symmetric matrix $D$, independent of $X$:
    \begin{align}
        \label{2ordA}
    &\mathbb{E}\left[X^TCX\right]=\lambda^2\mbox{Tr}\left(C\right)I_p\\
    &\mathbb{E}\left[X^TXDX^TX\right]=\lambda^4n\left(n+1\right)D+\lambda^4n\mbox{Tr}\left(D\right)I_p
    \end{align}

\end{lemma}

\begin{proof}

The Lemma follows by direct computations of the above expectations, using Isserlis' theorem. Particularly, for higher order exponents, combinatorics plays a crucial role in counting products of different Gaussian variables in an effective way.

\end{proof}

\begin{lemma} 

\label{lemma2}

Let $X_i^v$, $X_i^t$ be Gaussian random matrices, of size respectively $n_v \times p$ and $n_t \times p$, with independent rows, and each row has covariance equal to $I_p$, the $p\times p$ identity matrix. Let $p(\xi)$ and $n_t(\xi)$ be any function of order $O(\xi)$ as $\xi\rightarrow\infty$. Then:
    \begin{align}
    &X_i^v{X_i^v}^T=\lambda^2p\;I_{n_v}+O\left(\xi^{1/2}\right)\\
    &X_i^v{X_i^t}^TX_i^t{X_i^v}^T=\lambda^4pn_t\;I_{n_v}+O\left(\xi^{3/2}\right)\\
    &X_i^v{X_i^t}^TX_i^t{X_i^t}^TX_i^t{X_i^v}^T=\lambda^6pn_t(n_t+p+1)I_{n_v}+O\left(\xi^{5/2}\right)
    \end{align}
Note that the order $O\left(\xi\right)$ applies to all elements of the matrix in each expression.
For $i\neq j$
    \begin{align}
    &X_i^v{X^{v(j)}}^T=O\left(\xi^{1/2}\right)\\
    &X_i^v{X_i^t}^TX_i^t{X^{v(j)}}^T=O\left(\xi^{3/2}\right)\\
    &X_i^v{X_i^t}^TX_i^t{X^{t(j)}}^TX^{t(j)}{X^{v(j)}}^T=O\left(\xi^{5/2}\right)
    \end{align}
Furthermore, for any positive real number $\delta$ and for any $p\times p$ symmetric matrix $D$ independent of X, where Tr$(D)$ and Tr$(D^2)$ are both of order $O(\xi^\delta)$
    \begin{align}
    &X_i^vD{X_i^v}^T=\lambda^2\mbox{Tr}\left(D\right)I_{n_v}+O\left(\xi^{\delta/2}\right)\\
    &X_i^v{X_i^t}^TX_i^tD{X_i^v}^T=\lambda^4\mbox{Tr}\left(D\right)n_tI_{n_v}+O\left(\xi^{1+\delta/2}\right)\\
    &X_i^v{X_i^t}^TX_i^tD{X_i^t}^TX_i^t{X_i^v}^T=\nonumber\\
    &\lambda^6\mbox{Tr}\left(D\right)n_t(n_t+p+1)I_{n_v}+O\left(\xi^{2+\delta/2}\right)\\
    &X_i^vD{X^{v(j)}}^T=O\left(\xi^{\delta/2}\right)\\
    &X_i^v{X_i^t}^TX_i^tD{X^{v(j)}}^T=O\left(\xi^{1+\delta/2}\right)\\
    &X_i^v{X_i^t}^TX_i^tD{X^{t(j)}}^TX^{t(j)}{X^{v(j)}}^T=O\left(\xi^{2+\delta/2}\right)
    \end{align}

\end{lemma}

\begin{proof}
The Lemma follows by direct computations of the expectations and variances of each term.

\end{proof}

\begin{lemma} 

\label{lemma3}

Let $X^{v}$, $X^{t}$ be Gaussian random matrices, of size respectively $n_v \times p$ and $n_t \times p$, with independent rows, and each row has covariance equal to $I_p$, the $p\times p$ identity matrix. Let $n_v(\xi)$ and $n_t(\xi)$ be any function of order $O(\xi)$ for $\xi\rightarrow\infty$. Then:
    \begin{align}
    &{X^{v}}^TX^{v}=\lambda^2n_v\;I_{p}+O\left(\xi^{1/2}\right)\\
    &{X^{t}}^TX^{t}{X^{v}}^TX^{v}=\lambda^4n_tn_v\;I_{p}+O\left(\xi^{3/2}\right)\\
    &{X^{t}}^TX^{t}{X^{v}}^TX^{v}{X^{t}}^TX^{t}= \nonumber\\
&=\lambda^6n_vn_t(n_t+p+1)I_{p}+O\left(\xi^{5/2}\right)
    \end{align}
    Note that the order $O\left(\xi\right)$ applies to all elements of the matrix in each expression.

\end{lemma}

\begin{proof}
The Lemma follows by direct computations of the expectations and variances of each term.

\end{proof}

\subsection{Averaging over test data}

We calculate the average test loss as a function of the hyperparameters $n_i^t$, $n_i^v$, $n_r$, $p$, $m$, $\alpha_i$, $\alpha_r$, $\sigma_i$, $\sigma_r$, $\lambda_i$, $\lambda_r$, $\nu$, $\mathbf{w}_0$.
The subscript $i$ denotes meta-training hyperparameters for task $i$, while the subscript $r$ denotes meta-testing hyperparameters.
Using the expression in Eq.\ref{testoutput} for the test output, we rewrite the test loss in Eq.\ref{losstest} as
\begin{equation}
\overline{\mathcal{L}}^{test}=\mathop{\mathbb{E}}\frac{1}{2n_s}\left|X^{s}\left(\mathbf{w'}-\boldsymbol\theta^\star\right)+\mathbf{z}^{s}\right|^2
\end{equation}
We start by averaging this expression with respect to $X^s, \mathbf{z}^s$, noting that $\boldsymbol\theta^\star$ does not depend on test data.
We further average with respect to $\mathbf{w'}$, but note that $\boldsymbol\theta^\star$ depends on test parameters, so we average only terms that do not depend on $\boldsymbol\theta^\star$.
Using Eq.\ref{1ord}, the result is

\label{losstest2}
\begin{align}
    &
\overline{\mathcal{L}}^{test}=\frac{\sigma_r^2}{2}+\frac{\lambda_r^2}{2}\left(\nu^2+\left|\mathbf{w}_0\right|^2\right)+\nonumber\\
&+\lambda_r^2\mathbb{E}\left[\frac{\left|\boldsymbol\theta^\star\right|^2}{2}-\left(\mathbf{w}_0+\delta\mathbf{w'}\right)^T\boldsymbol\theta^\star\right]
\end{align}\\

where we define $\delta\mathbf{w'}=\mathbf{w'}-\mathbf{w}_0$.
The second term in the expectation is linear in $\boldsymbol\theta^\star$ and can be averaged over $X^r, \mathbf{z}^r$, using Eq.\ref{thetastar} and noting that  $\boldsymbol\omega^\star$ does not depend on target data.
The result is
\begin{equation}
\label{Ethetastar}
\mathop{\mathbb{E}}_{X^r}\mathop{\mathbb{E}}_{\mathbf{z}^r}\;\boldsymbol\theta^\star=(1-\lambda_r^2\alpha_r)\boldsymbol\omega^\star+\lambda_r^2\alpha_r\left(\mathbf{w}_0+\delta\mathbf{w'}\right)
\end{equation}
Using Eq.\ref{Ethetastar} we average over $\mathbf{w'}$ the second term in the expectation of Eq.\ref{losstest2} and find
\label{losstest3}
\begin{align}
&\overline{\mathcal{L}}^{test}=\frac{\sigma_r^2}{2}+\lambda_r^2\left(\frac{1}{2}-\lambda_r^2\alpha_r\right)\left(\nu^2+\left|\mathbf{w}_0\right|^2\right)-\nonumber\\
&-\lambda_r^2\left(1-\lambda_r^2\alpha_r\right)\mathbf{w}_0^T\mathbb{E}\;\boldsymbol\omega^\star+\lambda^2\mathbb{E}\frac{\left|\boldsymbol\theta^\star\right|^2}{2}
\end{align}
We average the last term of this expression over $\mathbf{z}^r, \mathbf{w'}$, using Eq.\ref{thetastar} and noting that $\boldsymbol\omega^\star$ does not depend on target data and test parameters.
The result is
\begin{align}
&\mathop{\mathbb{E}}_{\mathbf{w'}}\mathop{\mathbb{E}}_{\mathbf{z}^r}\left|\boldsymbol\theta^\star\right|^2=\left|\boldsymbol\omega^\star\right|^2+\frac{\alpha_r^2}{n_r^2}\left(\boldsymbol\omega^\star-\mathbf{w}_0\right)^T\left({X^{r}}^TX^{r}\right)^2\left(\boldsymbol\omega^\star-\mathbf{w}_0\right)-\nonumber\\
&-\frac{2\alpha_r}{n_r}{\boldsymbol\omega^\star}^T{X^{r}}^TX^{r}\nonumber\\
&\left(\boldsymbol\omega^\star-\mathbf{w}_0\right)+\frac{\alpha_r^2\sigma_r^2}{n_r^2}\mbox{Tr}\left[{X^{r}}{X^{r}}^T\right]+\frac{\alpha_r^2\nu^2}{n_r^2p}\mbox{Tr}\left[\left({X^{r}}{X^{r}}^T\right)^2\right]
\end{align}
We now average over $X^r$, again noting that $\boldsymbol\omega^\star$ does not depend on target data.
Using Eqs.\ref{1ord}, \ref{2ord}, we find
\begin{align}
&\mathop{\mathbb{E}}_{X^r}\mathop{\mathbb{E}}_{\mathbf{w'}}\mathop{\mathbb{E}}_{\mathbf{z}^r}\left|\boldsymbol\theta^\star\right|^2=\left|\boldsymbol\omega^\star\right|^2+\lambda_r^4\alpha_r^2\left(1+\frac{p+1}{n_r}\right)\nonumber\\
&\left(\nu^2+\left|\boldsymbol\omega^\star-\mathbf{w}_0\right|^2\right)-2\lambda_r^2\alpha_r{\boldsymbol\omega^\star}^T\left(\boldsymbol\omega^\star-\mathbf{w}_0\right)+\frac{\lambda_r^2\alpha_r^2\sigma_r^2p}{n_r}
\end{align}
We can now rewrite the average test loss \ref{losstest3} as
\label{losstest4}
\begin{align}
&\overline{\mathcal{L}}^{test}=\frac{\sigma_r^2}{2}\left(1+\frac{\lambda_r^4\alpha_r^2p}{n_r}\right)+\nonumber\\
&+ \frac{\lambda_r^2}{2}\left[\left(1-\lambda_r^2\alpha_r\right)^2+\lambda_r^4\alpha_r^2\frac{p+1}{n_r}\right]\left(\nu^2+\mathbb{E}\left|\boldsymbol\omega^\star-\mathbf{w}_0\right|^2\right)
\end{align}
In order to average the last term, we need an expression for $\boldsymbol\omega^\star$.
We note that the loss in Eq.\ref{lossmeta} is quadratic in $\boldsymbol\omega$, therefore the solution of Eq.\ref{omegastar} can be found using standard linear algebra.
In particular, the loss in Eq.\ref{lossmeta} can be rewritten as
\begin{equation}
\label{lossmetashort}
\mathcal{L}^{meta}=\frac{1}{2m}\left|\boldsymbol\gamma-B\boldsymbol\omega\right|^2
\end{equation}
where $\boldsymbol\gamma$ is a column vector of $\sum_{i=1}^mn_i^v$ components, and $B$ is a matrix of shape $\sum_{i=1}^mn_i^v\times p$.
The vector $\boldsymbol\gamma$ is a stack of $m$ vectors
\begin{equation}
{\tiny
\boldsymbol\gamma=\left(
\begin{matrix}\frac{1}{\sqrt{n_1^v}}\left[X_1^v\left(I_p-\frac{\alpha_1}{n_1^t}{X_1^t}^TX_1^t\right)\mathbf{w}_1-\frac{\alpha_1}{n_1^t}X_1^v{X_1^t}^T\mathbf{z}_1^t+\mathbf{z}_1^v\right]\\ 
\vdots\\
\frac{1}{\sqrt{n_m^v}}\left[X_m^v\left(I_p-\frac{\alpha_m}{n_m^t}{X_m^t}^TX_m^t\right)\mathbf{w}_m-\frac{\alpha_m}{n_m^t}X_m^v{X_m^t}^T\mathbf{z}_m^t+\mathbf{z}_m^v\right]
\end{matrix}\right)
}
\end{equation}
Similarly, the matrix $B$ is a stack of $m$ matrices
\begin{equation}
\label{Bmatrix}
B=\left(
\begin{matrix}\frac{1}{\sqrt{n_1^v}}\left[X_1^v\left(I_p-\frac{\alpha_1}{n_1^t}{X_1^t}^TX_1^t\right)\right]\\ 
\vdots\\
\frac{1}{\sqrt{n_m^v}}\left[X_m^v\left(I_p-\frac{\alpha_m}{n_m^t}{X_m^t}^TX_m^t\right)\right]
\end{matrix}\right)
\end{equation}
We denote by $I_p$ the $p\times p$ identity matrix.
The expression for $\boldsymbol\omega$ that minimizes Eq.\ref{lossmetashort} depends on whether the problem is overparameterized ($p>\sum_{i=1}^mn_i^v$) or underparameterized ($p<\sum_{i=1}^mn_i^v$), therefore we distinguish these two cases in the following sections.

\subsection{Averaging over training data: overparameterized case}
\label{sec:overp}

In the overparameterized case ($p>\sum_{i=1}^mn_i^v$), under the assumption that the inverse of $BB^T$ exists, the value of $\boldsymbol\omega$ that minimizes Eq.\ref{lossmetashort} is equal to
\begin{equation}
\label{omegastar_op}
\boldsymbol\omega^\star=B^T\left(BB^T\right)^{-1}\boldsymbol\gamma+\left[I_p-B^T\left(BB^T\right)^{-1}B\right]\boldsymbol\omega_0
\end{equation}
The vector $\boldsymbol\omega_0$ is interpreted as the initial condition of the parameter optimization of the outer loop, when optimized by gradient descent.
Note that the matrix $B$ does not depend on $\mathbf{w}, \mathbf{z}^t, \mathbf{z}^v$, and
$\mathop{\mathbb{E}}_{\mathbf{w}}\mathop{\mathbb{E}}_{\mathbf{z}^t}\mathop{\mathbb{E}}_{\mathbf{z}^v}\;\boldsymbol\gamma=B\mathbf{w}_0$.
We denote by $\delta\boldsymbol\gamma$ the deviation from the average, and we have
\begin{align}
&\boldsymbol\omega^\star-\mathbf{w}_0=B^T\left(BB^T\right)^{-1}\delta\boldsymbol\gamma+\nonumber\\
&+\left[I_p-B^T\left(BB^T\right)^{-1}B\right]\left(\boldsymbol\omega_0-\mathbf{w}_0\right)
\end{align}
We square this expression and average over $\mathbf{w}, \mathbf{z}^t, \mathbf{z}^v$.
We use the cyclic property of the trace and the fact that $B^T\left(BB^T\right)^{-1}B$ is a projection.
The result is
\label{exprSquarediff}
\begin{align}
&\left|\boldsymbol\omega^\star-\mathbf{w}_0\right|^2=\mbox{Tr}\left[\Gamma\left(BB^T\right)^{-1}\right]+\nonumber\\
&+\left(\boldsymbol\omega_0-\mathbf{w}_0\right)^T\left[I_p-B^T\left(BB^T\right)^{-1}B\right]\left(\boldsymbol\omega_0-\mathbf{w}_0\right)
\end{align}
The matrix $\Gamma$ is defined as
\begin{equation}
\label{Gamma1}
\Gamma=\mathop{\mathbb{E}}_{\mathbf{w}}\mathop{\mathbb{E}}_{\mathbf{z}^t}\mathop{\mathbb{E}}_{\mathbf{z}^v}\delta\boldsymbol\gamma\;\delta\boldsymbol\gamma^T=\left(
\begin{matrix}
\Gamma^{(1)}&0&0\\0&\ddots&0\\0&0&\Gamma^{(m)}
\end{matrix}
\right)
\end{equation}
Where matrix blocks are given by the following expression
\begin{align}
\label{Gamma2}
&\Gamma^{(i)}=\frac{\nu^2}{n_i^vp}X_i^v\left(I_p-\frac{\alpha_i}{n_i^t}{X_i^t}^TX_i^t\right)^2{X_i^v}^T+\nonumber\\
&+\frac{\sigma_i^2}{n_i^v}\left(I_{n_i^v}+\frac{\alpha_i^2}{{n_i^t}^2}X_i^v{X_i^t}^TX_i^t{X_i^v}^T\right)
\end{align}
It is convenient to rewrite the scalar product of Eq.\ref{exprSquarediff} in terms of the trace of outer products
\label{avgomegastar_op}
\begin{align}
&\left|\boldsymbol\omega^\star-\mathbf{w}_0\right|^2=\nonumber\\
&=\mbox{Tr}\left[\left(BB^T\right)^{-1}\left(\Gamma-B\left(\boldsymbol\omega_0-\mathbf{w}_0\right)\left(\boldsymbol\omega_0-\mathbf{w}_0\right)^TB^T\right)\right]+\nonumber\\
&+\left|\boldsymbol\omega_0-\mathbf{w}_0\right|^2
\end{align}
In order to calculate $\mathbb{E}\left|\boldsymbol\omega^\star-\mathbf{w}_0\right|^2$ in Eq.\ref{losstest4}  we need to average this expression over training and validation data.
These averages are hard to compute since they involve nonlinear functions of the data.
However, we can approximate these terms by assuming that $p$ and $n_i^t$ are large, both of order $O(\xi)$, where $\xi$ is a large number.
Furthermore, we assume that $\left|\boldsymbol\omega_0-\mathbf{w}_0\right|$ is of order $O(\xi^{-1/4})$.
Using Lemma \ref{lemma2}, together with the expressions of $B$ (Eq.\ref{Bmatrix}) and $\Gamma$ (Eqs.\ref{Gamma1},\ref{Gamma2}), we can prove that
{\small
\begin{align}\label{BBT}
&\frac{1}{p}BB^T=\left(
\begin{matrix}
{n_1^v}^{-1}{\lambda_1}^2h_1I_{n_1^v}&0&0\\0&\ddots&0\\0&0&{n_m^v}^{-1}{\lambda_m}^2h_mI_{n_m^v}
\end{matrix}
\right)+ \nonumber\\
&+O\left(\xi^{-1/2}\right)
\end{align} 
}
\begin{align}
&\Gamma^{(i)}=\left\{\frac{\nu^2\lambda_i^2h_i}{n_i^v}+\frac{\sigma_i^2}{n_i^v}\left(1+\frac{\lambda_i^4\alpha_i^2p}{n_i^t}\right)\right\}I_{n_i^v}+\nonumber\\
&+O\left(\xi^{-1/2}\right)
\end{align}
{\small
\begin{align}
&B\left(\boldsymbol\omega_0-\mathbf{w}_0\right)\left(\boldsymbol\omega_0-\mathbf{w}_0\right)^TB^T=\nonumber\\
&=\left|\boldsymbol\omega_0-\mathbf{w}_0\right|^2\left(
\begin{matrix}
{{n_1^v}^{-1}\lambda_1}^2h_1I_{n_1^v}&0&0\\0&\ddots&0\\0&0&{n_m^v}^{-1}{\lambda_m}^2h_mI_{n_m^v}
\end{matrix}\right)+\nonumber\\
&O\left(\xi^{-1/2}\right)
\end{align}
}
where we define the following expression
\begin{align}
h_i=\left(1-\lambda_i^2\alpha_i\right)^2+\lambda_i^4\alpha_i^2\frac{p+1}{n_i^t}
\end{align}
Using Eq.\ref{BBT} and Taylor expansion, the inverse $\left(BB^T\right)^{-1}$ is equal to
{\tiny
\begin{align}
&\left(BB^T\right)^{-1}=\frac{1}{p}\left(
\begin{matrix}
n_1^v{\lambda_1}^{-2}{h_1}^{-1}I_{n_1^v}&0&0\\0&\ddots&0\\0&0&n_m^v{\lambda_m}^{-2}{h_m}^{-1}I_{n_m^v}
\end{matrix}
\right)+\nonumber\\
&+O\left(\xi^{-3/2}\right),
\end{align}
}
Substituting the three expressions above in Eq.\ref{avgomegastar_op}, and ignoring terms of lower order, we find
\begin{align}
&\mathbb{E}\left|\boldsymbol\omega^\star-\mathbf{w}_0\right|^2=\left(1-\frac{1}{p}\sum_{i=1}^mn_i^v\right)\left|\boldsymbol\omega_0-\mathbf{w}_0\right|^2+\nonumber\\
&+\frac{\nu^2}{p}\sum_{i=1}^mn_i^v+\frac{1}{p}\sum_{i=1}^m\frac{\sigma_i^2n_i^v}{\lambda_i^2h_i}\left(1+\frac{\lambda_i^4\alpha_i^2p}{n_i^t}\right)+O\left(\xi^{-3/2}\right)
\end{align}
Substituting this expression into in Eq.\ref{losstest4}, we find the value of average test loss
\begin{align}
&\overline{\mathcal{L}}^{test}=\frac{\sigma_r^2}{2}\left(1+\frac{\lambda_r^4\alpha_r^2p}{n_r}\right)+\nonumber\\
&+\frac{\lambda_r^2h_r}{2}\left(1-\frac{1}{p}\sum_{i=1}^mn_i^v\right)\left|\boldsymbol\omega_0-\mathbf{w}_0\right|^2+\nonumber\\
&+\frac{\lambda_r^2h_r\nu^2}{2}\left(1+\frac{1}{p}\sum_{i=1}^mn_i^v\right)+\nonumber\\
&+\frac{\lambda_r^2h_r}{2p}\sum_{i=1}^m\frac{\sigma_i^2n_i^v}{\lambda_i^2h_i}\left(1+\frac{\lambda_i^4\alpha_i^2p}{n_i^t}\right)+\nonumber\\
&+O\left(\xi^{-3/2}\right)
\end{align}
where we define the following expression
\begin{align}
h_r=\left(1-\lambda_r^2\alpha_r\right)^2+\lambda_r^4\alpha_r^2\frac{p+1}{n_r}
\end{align}

\subsection{Averaging over training data: underparameterized case} 

In the underparameterized case ($p<\sum_{i=1}^mn_i^v$), under the assumption that the inverse of $B^TB$ exists, the value of $\boldsymbol\omega$ that minimizes Eq.\ref{lossmetashort} is equal to
\begin{equation}
\label{omegastar_up}
\boldsymbol\omega^\star=\left(B^TB\right)^{-1}B^T\boldsymbol\gamma
\end{equation}
Note that the matrix $B$ does not depend on $\mathbf{w}, \mathbf{z}^t, \mathbf{z}^v$, and
$\mathop{\mathbb{E}}_{\mathbf{w}}\mathop{\mathbb{E}}_{\mathbf{z}^t}\mathop{\mathbb{E}}_{\mathbf{z}^v}\;\boldsymbol\gamma=B\mathbf{w}_0$.
We denote by $\delta\boldsymbol\gamma$ the deviation from the average, and we have
\begin{equation}
\left|\boldsymbol\omega^\star-\mathbf{w}_0\right|^2=\mbox{Tr}\left[\left(B^TB\right)^{-1}B^T\delta\boldsymbol\gamma\;\delta\boldsymbol\gamma^TB\left(B^TB\right)^{-1}\right]
\end{equation}
We need to average this expression in order to calculate $\mathbb{E}\left|\boldsymbol\omega^\star-\mathbf{w}_0\right|^2$ in Eq.\ref{losstest4}.
We start by averaging $\delta\boldsymbol\gamma\;\delta\boldsymbol\gamma^T$ over $\mathbf{w}, \mathbf{z}^t, \mathbf{z}^v$, since $B$ does not depend on those variables.
Note that $\mathbf{w}, \mathbf{z}^t, \mathbf{z}^v$ are independent on each other and across tasks.
As in previous section, we denote by $\Gamma$ the result of this operation, given by Eq.s\ref{Gamma1}, \ref{Gamma2}.
Finally, we need to average over the training and validation data
\begin{equation}
\label{avgomegastar_up}
\mathbb{E}\left|\boldsymbol\omega^\star-\mathbf{w}_0\right|^2=\mathop{\mathbb{E}}_{X^t}\mathop{\mathbb{E}}_{X^v}\mbox{Tr}\left[\left(B^TB\right)^{-1}B^T\Gamma B\left(B^TB\right)^{-1}\right]
\end{equation}
It is hard to average this expression because it includes nonlinear functions of the data.
However, we can approximate these terms by assuming that either $m$ or $\xi$ (or both) is a large number, where $\xi$ is defined by assuming that $n_i^t$ and $n_i^v$ are of order $O(\xi)$.
Using Lemma \ref{lemma3}, together with the expression of $B$ (Eq.\ref{Bmatrix}), and noting that each factor in Eq.\ref{avgomegastar_up} has a sum over $m$ independent terms, we can prove that
\begin{equation}
B^TB=\sum_{i=1}^m\lambda_i^2h_iI_p+O\left((m\xi)^{1/2}\right)
\end{equation}
Using this result and a Taylor expansion, the inverse is equal to
\begin{equation}
\left(B^TB\right)^{-1}=\left[\sum_{i=1}^m\lambda_i^2h_i\right]^{-1}I_p+O\left((m\xi)^{-3/2}\right)
\end{equation}
Similarly, the term $B^T\Gamma B$ is equal to its average plus a term of smaller order
\begin{equation}
B^T\Gamma B=\mathbb{E}\left(B^T\Gamma B\right)+O\left((m\xi)^{1/2}\right)
\end{equation}
We substitute these expressions in Eq.\ref{avgomegastar_up} and neglect lower orders.
Here we show how to calculate explicitly the expectation of $B^T\Gamma B$.
For ease of notation, we define the matrix $A_i^t=I-\frac{\alpha_i}{n_i^t}{X_i^t}^TX_i^t$.
Using the expressions of $B$ (Eq.\ref{Bmatrix}) and $\Gamma$ (Eqs.\ref{Gamma1},\ref{Gamma2}), the expression for $B^T\Gamma B$ is given by
\begin{align}
\label{BGB}
&B^T\Gamma B=\sum_{i=1}^m{n_i^v}^{-2}\sigma_i^2{A_i^t}^T{X_i^v}^TX_i^vA_i^t+\nonumber\\
&+\frac{\nu^2}{p}\sum_{i=1}^m{n_i^v}^{-2}\left({A_i^t}^T{X_i^v}^TX_i^vA_i^t\right)^2+\nonumber\\
&+\sum_{i=1}^m{n_i^v}^{-2}{n_i^t}^{-2}\alpha_i^2\sigma_i^2{A_i^t}^T{X_i^v}^TX_i^v{X_i^t}^TX_i^t{X_i^v}^TX_i^vA_i^t
\end{align}
We use Eqs.\ref{1ord}, \ref{2ord} to calculate the average of the first term in Eq.\ref{BGB}
\begin{align}
\mathop{\mathbb{E}}_{X^t}\mathop{\mathbb{E}}_{X^v}\sum_{i=1}^m{n_i^v}^{-2}\sigma_i^2{A_i^t}^T{X_i^v}^TX_i^vA_i^t=\sum_{i=1}^m{n_i^v}^{-1}\lambda_i^2\sigma_i^2h_iI_p
\end{align}
We use Eqs.\ref{1ord}, \ref{2ord}, \ref{3ord}, \ref{2ordA}, \ref{2ordTr}, \ref{3ordTr}, \ref{3ordTr2}, \ref{4ordTr} to calculate the average of the second term
\begin{align}
\label{E2}
&\mathop{\mathbb{E}}_{X^t}\mathop{\mathbb{E}}_{X^v}\sum_{i=1}^m{n_i^v}^{-2}\left({A_i^t}^T{X_i^v}^TX_i^vA_i^t\right)^2=\nonumber\\
&=\mathop{\mathbb{E}}_{X^t}\sum_{i=1}^m{n_i^v}^{-1}\lambda_i^4\left[\left(n_i^v+1\right){A_i^t}^4+{A_i^t}^2\mbox{Tr}\left({A_i^t}^2\right)\right]=\nonumber\\
&=\sum_{i=1}^m{n_i^v}^{-1}\lambda_i^4\left(n_i^v+1\right)\bigl(1-4\lambda_i^2\alpha_i+6\lambda_i^4\alpha_i^2\mu_2^{(i)}-\nonumber\\
&-4\lambda_i^6\alpha_i^3\mu_3^{(i)}+\lambda_i^8\alpha_i^4\mu_4^{(i)}\bigr)I_p+\nonumber\\
&+\sum_{i=1}^m{n_i^v}^{-1}\lambda_i^4p\bigl(1-4\lambda_i^2\alpha_i+2\lambda_i^4\alpha_i^2\mu_2^{(i)}4\lambda_i^4\alpha_i^2\mu_{1,1}^{(i)}-\nonumber\\
&-4\lambda_i^6\alpha_i^3\mu_{2,1}^{(i)}+\lambda_i^8\alpha_i^4\mu_{2,2}^{(i)}\bigr)I_p
\end{align}
where we have defined
\begin{align}
&\mu_2^{(i)}=\frac{1}{n_i^t}\bigl(n_i^t+p+1\bigr)\\
&\mu_3^{(i)}=\frac{1}{{n_i^t}^2}\bigl({n_i^t}^2+p^2+3n_i^tp+3n_i^t+3p+4\bigr)\\
&\mu_4^{(i)}=\frac{1}{{n_i^t}^3}\bigl({n_i^t}^3+p^3+6{n_i^t}^2p+6{n_i^t}p^2+6{n_i^t}^2+6p^2+\nonumber\\
&+17n_i^tp+21n_i^t+21p+20\bigr)\\
&\mu_{1,1}^{(i)}=\frac{1}{{n_i^t}^2p}\bigl({n_i^t}^2p+2n_i^t\bigr)\\
&\mu_{2,1}^{(i)}=\frac{1}{{n_i^t}^2p}\bigl({n_i^t}^2p+{n_i^t}p^2+n_i^tp+4n_i^t+4p+4\bigr)\\
&\mu_{2,2}^{(i)}=\frac{1}{{n_i^t}^3p}\bigl({n_i^t}^3p+n_i^tp^3+2{n_i^t}^2p^2+2{n_i^t}^2p+\nonumber\\
&+2n_i^tp^2+8{n_i^t}^2+8p^2+21n_i^tp+20n_i^t+20p+20\bigr)
\end{align}
Finally, we compute the average of the third term, using Eqs.\ref{1ord}, \ref{2ord}, \ref{3ord}, \ref{4ord}, \ref{2ordA}, \ref{2ordTr}, \ref{3ordTr}
{\small
\begin{align}
&\mathop{\mathbb{E}}_{X^t}\mathop{\mathbb{E}}_{X^v}\sum_{i=1}^m{n_i^v}^{-2}{n_i^t}^{-2}\alpha_i^2\sigma_i^2{A_i^t}^T{X_i^v}^TX_i^v{X_i^t}^TX_i^t{X_i^v}^TX_i^vA_i^t=\nonumber\\
&=\mathop{\mathbb{E}}_{X^t}\sum_{i=1}^m\frac{\lambda_i^4\alpha_i^2\sigma_i^2}{n_i^v{n_i^t}^2}\bigl[\bigl(n_i^v+1\bigr){A_i^t}^T{X_i^t}^TX_i^tA_i^t+\nonumber\\
&+{A_i^t}^TA_i^t\mbox{Tr}\bigl({X_i^t}^TX_i^t\bigr)\bigr]=\nonumber\\
\label{E3}
&=\sum_{i=1}^m\frac{\lambda_i^6\alpha_i^2\sigma_i^2}{n_i^vn_i^t}\bigl[\bigl(n_i^v+1\bigr)\bigl(1-2\lambda_i^2\alpha_i\mu_2^{(i)}+\lambda_i^4\alpha_i^2\mu_3^{(i)}\bigr)+\nonumber\\
&+p\bigl(1-2\lambda_i^2\alpha_i\mu_{1,1}^{(i)}+\lambda_i^4\alpha_i^2\mu_{2,1}^{(i)}\bigr)\bigr]I_p
\end{align}
}
Putting everything together in Eq.\ref{avgomegastar_up}, and applying the trace operator, we find the following expression for the meta-parameter variance
\begin{align}
\label{Eomegasquare}
&\mathbb{E}\left|\boldsymbol\omega^\star-\mathbf{w}_0\right|^2=p\Bigl[\sum_{i=1}^m\lambda_i^2h_i\Bigr]^{-2}\sum_{i=1}^m\frac{\lambda_i^2}{n_i^v}\Bigg\{\sigma_i^2h_i+\nonumber\\
&+\frac{\lambda_i^4\alpha_i^2\sigma_i^2}{n_i^t}\Bigl[\Bigl(n_i^v+1\Bigr)\bigl(1-2\lambda_i^2\alpha_i\mu_2^{(i)}+\lambda_i^4\alpha_i^2\mu_3^{(i)}\Bigr)+\nonumber\\
&+p\Bigl(1-2\lambda_i^2\alpha_i\mu_{1,1}^{(i)}+\lambda_i^4\alpha_i^2\mu_{2,1}^{(i)}\Bigr)\Bigr]+\nonumber\\
&+\frac{\nu^2}{p}\lambda_i^2\bigg[\left(n_i^v+1\right)\Bigl(1-4\lambda_i^2\alpha_i+6\lambda_i^4\alpha_i^2\mu_2^{(i)}-\nonumber\\
&-4\lambda_i^6\alpha_i^3\mu_3^{(i)}+\lambda_i^8\alpha_i^4\mu_4^{(i)}\Bigr)+\nonumber\\
&+p\Bigl(1-4\lambda_i^2\alpha_i+2\lambda_i^4\alpha_i^2\mu_2^{(i)}+4\lambda_i^4\alpha_i^2\mu_{1,1}^{(i)}-\nonumber\\
&- 4\lambda_i^6\alpha_i^3\mu_{2,1}^{(i)}+\lambda_i^8\alpha_i^4\mu_{2,2}^{(i)}\Bigr)\bigg]\Bigg\}+O\left((m\xi)^{-3/2}\right)
\end{align}
We rewrite this expression as
\begin{align}
&\mathbb{E}\left|\boldsymbol\omega^\star-\mathbf{w}_0\right|^2=p\left[\sum_{i=1}^m\lambda_i^2h_i\right]^{-2}\sum_{i=1}^m\frac{\lambda_i^2}{n_i^v}\Bigg\{\nonumber\\
&\sigma_i^2\left[h_i+\frac{\lambda_i^4\alpha_i^2}{n_i^t}\left[\left(n_i^v+1\right)g_{1i}+pg_{2i}\right]\right]+\nonumber\\
&+\frac{\nu^2}{p}\lambda_i^2\left[\left(n_i^v+1\right)g_{3i}+pg_{4i}\right]\Bigg\}+O\left((m\xi)^{-3/2}\right)
\end{align}
where we defined the following expressions for $g_i$
\begin{align}
&g_{1i}=1-2\lambda_i^2\alpha_i\mu_2^{(i)}+\lambda_i^4\alpha_i^2\mu_3^{(i)}\\
&g_{2i}=1-2\lambda_i^2\alpha_i\mu_{1,1}^{(i)}+\lambda_i^4\alpha_i^2\mu_{2,1}^{(i)}\\
&g_{3i}=1-4\lambda_i^2\alpha_i+6\lambda_i^4\alpha_i^2\mu_2^{(i)}-4\lambda_i^6\alpha_i^3\mu_3^{(i)}+\nonumber\\
&+\lambda_i^8\alpha_i^4\mu_4^{(i)}\\
&g_{4i}=1-4\lambda_i^2\alpha_i+2\lambda_i^4\alpha_i^2\mu_2^{(i)}+4\lambda_i^4\alpha_i^2\mu_{1,1}^{(i)}-\nonumber\\
&-4\lambda_i^6\alpha_i^3\mu_{2,1}^{(i)}+\lambda_i^8\alpha_i^4\mu_{2,2}^{(i)}
\end{align}
Substituting this expression back into Eq.\ref{losstest4} returns the final expression for the average test loss, equal to
\begin{align}\label{gpoly1}
&\overline{\mathcal{L}}^{test}=\frac{\sigma_r^2}{2}\left(1+\frac{\lambda_r^4\alpha_r^2p}{n_r}\right)+\frac{\lambda_r^2h_r\nu^2}{2}+\nonumber\\
&+\frac{\lambda_r^2h_rp}{2}\left[\sum_{i=1}^m\lambda_i^2h_i\right]^{-2}\sum_{i=1}^m\frac{\lambda_i^2}{n_i^v}\Bigg\{\nonumber\\
&\sigma_i^2\left[h_i+\frac{\lambda_i^4\alpha_i^2}{n_i^t}\left[\left(n_i^v+1\right)g_{1i}+pg_{2i}\right]\right]+\nonumber\\
&\frac{\nu^2}{p}\lambda_i^2\left[\left(n_i^v+1\right)g_{3i}+pg_{4i}\right]\Bigg\}+O\left((m\xi)^{-3/2}\right)
\end{align}
%


\section{Optimal uniform allocation in mixed linear regression}
\label{sec:optuniall}

In order to prove Theorem \ref{optallocThm} in the main text, under the assumptions of Theorem \ref{underThm}, we start by rewriting Equation (\ref{testloss}) in the case of homogeneous tasks ($\sigma_i=\sigma$, $\lambda_i=\lambda$), fixed learning rate ($\alpha_i=\alpha$) and uniform allocation ($n_i=n$).
That is similar to the expression of Theorem 2 in \cite{bernacchia_meta-learning_2021}.
In this case, the budget is equal to $b=2nm$.
We further assume that, in addition to $n$ and $m$, $p$ is also large, of order $\Theta(\xi)$.
Then the average test loss is equal to

\begin{align}
&\overline{\mathcal{L}}^{test}=C_1+\frac{C_2p}{b}g_2^{-2}\left\{\sigma'^2\left[g_2+\alpha'^2\left(g_1+\frac{p}{n}g_2\right)\right]+\right.\nonumber\\
\label{Ltestsimp}
&\left.+\nu^2\left[\frac{n}{p}g_3+g_4\right]\right\}+O\left(\xi^{-2}\right)
\end{align}
where $\sigma'=\frac{\sigma}{\lambda}$, $\alpha'=\lambda^2\alpha$, and we have defined $g_1, g_2, g_3, g_4$ as
\begin{align}
&g_1=\left(1-\alpha'\right)^2-2\alpha'\frac{p}{n}+\alpha'^2\left(3\frac{p}{n}+\frac{p^2}{n^2}\right)\\
&g_2=\left(1-\alpha'\right)^2+\alpha'^2\frac{p}{n}\\
&g_3=\left(1-\alpha'\right)^4+6\alpha'^2\frac{p}{n}-\alpha'^3\left(12\frac{p}{n}+4\frac{p^2}{n^2}\right)+\\
&+\alpha'^4\left(6\frac{p}{n}+6\frac{p^2}{n^2}+\frac{p^3}{n^3}\right)\\
&g_4=\left(1-\alpha'\right)^4+2\alpha'^2\frac{p}{n}-4\alpha'^3\frac{p}{n}+\nonumber\\
&+\alpha'^4\left(2\frac{p}{n}+\frac{p^2}{n^2}\right)
\end{align}
and $C_1$, $C_2$ do not depend on $n$ and are equal to
\begin{align}
C_1=\frac{\sigma_r^2}{2}\left(1+\frac{\lambda_r^4\alpha_r^2p}{n_r}\right)+\frac{\nu^2}{2}C_2\\
C_2=\lambda_r^2\left(1-\lambda_r^2\alpha_r\right)^2+\lambda_r^6\alpha_r^2\frac{p}{n_r}
\end{align}
where the subscript $r$ denotes meta-testing hyperparameters.

We drop the term $O\left(\xi^{-2}\right)$, and we optimize the test loss in Eq.(\ref{Ltestsimp}) for the number of data points per task $n$, by taking the gradient of $\overline{\mathcal{L}}^{test}$ with respect to $n$ and set it to zero.
We obtain the following expression
\begin{align}
\frac{\partial \overline{\mathcal{L}}^{test}}{\partial n}=\frac{C_2p}{2b}\frac{\Gamma}{\left(\frac{ng_2}{p}\right)^3}
\end{align}
where $\Gamma$ is a cubic expression in $n/p$, equal to
\begin{align}
&\Gamma=\nu^2(1-\alpha')^6\frac{n^3}{p^3} + 3\nu^2\alpha'^2(1-\alpha')^4\frac{n^2}{p^2} +\\
&+ 2\alpha'^3\Big[\nu^2(2-\alpha'-4\alpha'^2+3\alpha'^3)+\\ &+\sigma'^2(2-5\alpha'+4\alpha'^2-\alpha'^3) \Big]\frac{n}{p} + \\ &+2\alpha'^4\left[\nu^2(2\alpha'^2-1) + \sigma'^2(2\alpha'-1) \right]=0
\end{align}
The zeros of the cubic represent stationary points, maxima or minima, of the loss as a function of $n$.
For ease of notation, we rewrite the cubic as
\begin{align}
&A\left(\frac{n}{p}\right)^3B+\left(\frac{n}{p}\right)^2 +C\left(\frac{n}{p}\right) +D=0
\end{align}
where the coefficients $A$, $B$, $C$, $D$ are defined as
\begin{align}
&A=\nu^2(1-\alpha')^6\\
&B=3\nu^2\alpha'^2(1-\alpha')^4\\
&C=2\alpha'^3\Big[\nu^2(2-\alpha'-4\alpha'^2+3\alpha'^3)+\\
&+\sigma'^2(2-5\alpha'+4\alpha'^2-\alpha'^3) \Big]\\
&D=2\alpha'^4\left[\nu^2(2\alpha'^2-1) + \sigma'^2(2\alpha'-1) \right]
\end{align}
For small $\alpha'$, these coefficients are of different orders, in particular $A$ is $O(1)$, $B$ is of order $O(\alpha'^2)$,  $C$ is of order $O(\alpha'^3)$,  $D$ is of order $O(\alpha'^4)$.
The discriminant of the cubic is equal to
\begin{align}
-27A^2D^2+18ABCD-4AC^3-4B^3D+B^2C^2
\end{align}
We arranged the terms in this expression according to their $\alpha'$ order, which is equal to, respectively, $O(\alpha'^8)$, $O(\alpha'^9)$, $O(\alpha'^9)$, $O(\alpha'^{10})$, $O(\alpha'^{10})$.
Note that the leading term is always negative, therefore there exists a value $\alpha'^\star$ such that, for $|\alpha'|<\alpha'^\star$, the discriminant of the cubic is negative, which means that the cubic has only one real solution.
This stationary point has to be a minimum, since the derivative of the average test loss exists everywhere (note that $g_2>0$) and it becomes large and positive for large positive or negative $n$. 

The cubic equation can be solved using standard methods, such as Cardano's method and depressed cubic transformation.
The solution is equal to
{\small
\begin{align}
\label{optimaln}
\nonumber
&n^\star=-\frac{p}{3A}\left[B+\left(\frac{-1+\sqrt{-3}}{2}\right)^k\left(\frac{\Delta_1+\sqrt{\Delta_1^2-4\Delta_0^3}}{2}\right)^{\frac{1}{3}}\right.\\
&\left.+\left(\frac{-1+\sqrt{-3}}{2}\right)^{-k}\left(\frac{2\Delta_0^3}{\Delta_1+\sqrt{\Delta_1^2-4\Delta_0^3}}\right)^{\frac{1}{3}}\right]
\end{align}
}
for $k=0,1,2$, where the real and positive solution among the three solutions should be selected, and $\Delta_0$, $\Delta_1$ are defined as
\begin{align}
&\Delta_0=B^2-3AC\\
&\Delta_1=2B^3-9ABC+27A^2D
\end{align}
Note that in the main text we defined the number of data points per task as $N=n_t+n_v=2n$, therefore we have $N^\star=2n^\star$ and
{\small
\begin{align}
\nonumber
&N^\star=-\frac{2p}{3A}\left[B+\left(\frac{-1+\sqrt{-3}}{2}\right)^k\left(\frac{\Delta_1+\sqrt{\Delta_1^2-4\Delta_0^3}}{2}\right)^{\frac{1}{3}}\right.\\
&\left.+\left(\frac{-1+\sqrt{-3}}{2}\right)^{-k}\left(\frac{2\Delta_0^3}{\Delta_1+\sqrt{\Delta_1^2-4\Delta_0^3}}\right)^{\frac{1}{3}}\right]
\label{Nsolution}
\end{align}
}
for $k=0,1,2$, where the real and positive solution among the three solutions should be taken.

The expression for $N^\star$ is rather complicated and it is hard to evaluate how the optimum depends on the hyperparameters of the model.
Therefore we computed an approximation that holds for small $|\alpha'|$.
Of the three terms in square brackets of Eq.(\ref{Nsolution}), the first one is of order $O(\alpha'^3)$, the second is $O(\alpha'^{4/3})$ and the third is $O(\alpha'^{5/3})$, thus we neglect the first and last term.
Furthermore, $\Delta_0^3$ is of order $O(\alpha'^9)$, while $\Delta_1^2$ is $O(\alpha'^8)$, thus we neglect $\Delta_0$, and we keep only the leading term in $\Delta_1$.
We obtain
\begin{align}
\label{nstarapprox}
N^\star\simeq 2\left[2\left(1+\frac{\sigma'^2}{\nu^2}\right)\right]^{\frac{1}{3}}\alpha'^{\frac{4}{3}}p
\end{align}

 \end{document}